\documentclass{article}

\usepackage[main,preprint]{neurips_2026}
\usepackage[utf8]{inputenc}
\usepackage[T1]{fontenc}
\usepackage{hyperref}
\usepackage{url}
\usepackage{booktabs}
\usepackage{amsfonts}
\usepackage{amsmath}
\usepackage{amssymb}
\usepackage{amsthm}
\usepackage{nicefrac}
\usepackage{microtype}
\usepackage{xcolor}
\usepackage{algorithm}
\usepackage{algpseudocode}
\usepackage{tikz}
\usetikzlibrary{positioning, arrows.meta, shapes.geometric, fit, backgrounds, calc, decorations.pathreplacing}
\usepackage{pgfplots}
\pgfplotsset{compat=1.17}
\usepackage{caption}
\usepackage{subcaption}

\makeatletter
\renewcommand\And{%
\end{tabular}\hskip 1em \@plus.17fil%
\begin{tabular}[t]{c}}
\makeatother

\newtheorem{definition}{Definition}
\newtheorem{proposition}{Proposition}
\newtheorem{remark}{Remark}

\title{SKG-Eval: Stateful Evaluation of Multi-Turn Dialogue via Incremental Semantic Knowledge Graphs}

\author{
	\textbf{Avijit Shil}$^{1}$ \hspace{0.25em}
	\textbf{Suman Samui}$^{2}$\\[1ex]
	$^{1}$Maulana Abul Kalam Azad University of Technology, West Bengal, India\\
	$^{2}$National Institute of Technology Durgapur, West Bengal, India\\[0.5ex]
	\texttt{avijitshil52460@gmail.com}, \hspace{0.25em}
	\texttt{ssamui@nitdgp.ac.in}
}

\begin{document}
	
	\maketitle
	
	\begin{abstract}

	Evaluating multi-turn dialogue systems poses fundamental challenges: the quality of each response depends not only on the immediate prompt but on a growing context of prior commitments, entities, and claims that the model is implicitly bound to. Existing automatic evaluators, including LLM-as-a-judge protocols and embedding-based metrics, largely operate on flat or turn-isolated representations, and therefore fail to reliably detect cross-turn failure modes such as contradiction, topic drift, and entity inconsistency.
	To address this, we propose \textbf{SKG-Eval}, a quasi-deterministic and interpretable evaluation framework that models dialogue as an evolving \emph{Semantic Knowledge Graph} (SKG) of entities, relations, and commitments across turns. At each turn, the graph is incrementally updated via structured triple extraction, and three complementary signals are computed: (i) \textit{local relevance}, measuring alignment with the current prompt and optional reference; (ii) \textit{historical consistency}, quantifying how newly introduced information connects to prior conversational state via graph-anchored and embedding-based signals; and (iii) \textit{logical coherence}, assessed by a geometric contradiction engine that detects cross-turn conflicts without relying on NLI models or LLM judges. These signals are fused using a regime-adaptive mechanism and aggregated into a length-invariant session score via recency-weighted trend analysis.
	Across multiple benchmarks, SKG-Eval achieves higher correlation with human judgments and substantially improves recall of long-range inconsistencies, particularly in extended conversations where existing evaluators degrade. In addition, SKG-Eval produces explicit contradiction certificates and yields deterministic scores given fixed inputs, enabling reproducible and auditable evaluation.
	Our results suggest that externalized state tracking via structured representations is a principled and scalable alternative to implicit reasoning in LLM-based evaluators for long-horizon dialogue systems.

	\end{abstract}
	
	\section{Introduction}
	\label{sec:intro}
	
	Large language models (LLMs) are increasingly deployed as multi-turn conversational agents in domains where the cost of a contradicted prior statement, a forgotten constraint, or a silently shifted topic ranges from user frustration to material harm. Yet the dominant paradigm for automatically evaluating dialogue still rests on \emph{turn-isolated} signals---a single response is scored against a single prompt and (sometimes) a reference, with conversational history compressed into a prefix that the evaluator is trusted to read \citep{lin2023llmeval,zhang2024comprehensive,mendonca2024ecoh}. This paradigm has known and increasingly visible failure modes. As soon as conversations grow beyond a few turns, models drop substantial competence \citep{kwan2024mteval,laban2025lostinconv}; they make assumptions in early turns that they later contradict \citep{li2026consistencylrm}; and frontier judges trained or prompted on short contexts do not reliably surface these errors at the session level \citep{sirdeshmukh2025multichallenge}.
	
	The root cause is that conversational quality is intrinsically \emph{stateful and temporal}. A turn that looks locally fluent and on-topic may nonetheless be wrong because, two turns earlier, the assistant claimed the opposite; or because it has slowly drifted off the user's actual question; or because it has substituted a new value for a fact already established. Capturing these failure modes requires the evaluator itself to maintain a structured representation of the conversation's commitments and to reason explicitly about how a new turn relates to that commitment store. LLM-as-a-judge mitigates this only partially: it pushes the burden of stateful reasoning onto a black-box judge whose attention pattern over long histories is itself unreliable, whose verdicts are non-deterministic, and whose contradiction recall on paraphrased and numerical conflicts is poor \citep{ike2025humanvsai}.
	
 We argue for an alternative: explicit, externalized state. We propose \textbf{SKG-Eval}, an evaluation framework that incrementally builds a typed, time-stamped \emph{Semantic Knowledge Graph} from the conversation as it unfolds, and scores each new turn against this graph rather than against a flat prefix. Three signals are extracted at every turn: a \emph{local relevance} score that triangulates the response against the prompt and (when available) a reference, a \emph{historical consistency} score that measures how the new turn's entities and facts attach to the existing graph, and a \emph{logical coherence} score produced by a purely geometric contradiction engine that compares the current turn's edges to its historical edges through negation flips, antonym pairs, numeric mismatches, and combined relation/object divergence. These signals are fused under a regime-adaptive weighting and aggregated across the session via a recency-weighted regression with a length-adaptive trend coefficient. 

To summarize our contributions:

\begin{itemize}
	
	\item \textbf{Stateful Dialogue Evaluation via Explicit Semantic Memory.} 
	We formulate multi-turn dialogue evaluation as reasoning over an evolving \textit{Semantic Knowledge Graph} (SKG) that explicitly represents entities, relations, and conversational commitments across turns. Through this externalized semantic state, cross-turn dependencies and long-range conversational consistency can be analyzed in a structured manner.
	
	\item \textbf{Geometric Contradiction Engine with Revision Awareness.} 
	A geometry-driven contradiction detection framework is proposed for graph-structured semantic representations. The engine detects inconsistencies through structured comparison of relations and objects, including negation reversals, antonymic relations, numeric mismatch, and relation-consistent object divergence. Revision-aware filtering is further incorporated to avoid penalizing legitimate conversational updates. As a result, interpretable contradiction certificates are produced without requiring NLI models or LLM judges during scoring.
	
	\item \textbf{Graph-Anchored Historical Consistency Modeling.} 
	We introduce a graph-anchored consistency metric that evaluates whether newly introduced information remains semantically connected to prior conversational state. A complementary session-anchor mechanism is also designed to capture higher-level thematic continuity across turns.
	
	\item \textbf{Robust Local Relevance through Multi-Signal Semantic Alignment.} 
	A triangulated relevance metric is developed that jointly considers prompt alignment and optional reference coverage, together with adaptive fallback mechanisms for short or reference-free responses. This design improves robustness relative to single-signal semantic similarity measures.
	
	\item \textbf{Regime-Adaptive Fusion and Session-Level Aggregation.} 
	We develop a regime-aware fusion strategy that dynamically weights relevance, consistency, and logical coherence according to response characteristics. At the session level, a recency-weighted regression-based aggregation mechanism is introduced to capture both quality level and temporal degradation trends across long conversations.
	
	\item \textbf{Interpretable and Quasi-Deterministic Evaluation.} 
	SKG-Eval provides explicit audit trails through per-turn scores, semantic anchors, and contradiction certificates. Empirically, strong alignment with human judgments is achieved while improved sensitivity to long-range conversational inconsistency and semantic drift is maintained across extended dialogue sessions.
	
\end{itemize}
	
	The remainder of the paper is organized as follows. Section~\ref{sec:related} surveys multi-turn evaluation, LLM-as-judge, dialogue coherence, and knowledge-graph-augmented evaluation, identifying the precise gap that SKG-Eval fills. Section~\ref{sec:problem} formalizes the problem and states the desiderata. Section~\ref{sec:method} develops the proposed framework component by component, with formal definitions, propositions, algorithmic descriptions, and complexity analysis. Section~\ref{sec:experiments} reports empirical results.
	
	\section{Related Work}
	\label{sec:related}
	
	\paragraph{Multi-turn dialogue benchmarks.} An initial suite of benchmarking tests has demonstrated that the single-turn estimation technique overstates dialogue abilities. MT-Bench \citep{zheng2023mtbench} proposed two-turn open prompts, where GPT-4 acts as a judge; nevertheless, its limited time horizon limits its ability to detect far-reaching failure cases. MT-Eval \citep{kwan2024mteval} expanded upon the former by considering four types of interactions—recollection, extension, refinement, and follow-ups—with significant drops in multi-turn performance unrelated to single-turn estimates, citing distance from relevant information and error propagation as key causes. MINT \citep{wang2024mint} evaluated tools usage and feedback handling in multi-turn iterations; MultiChallenge \citep{sirdeshmukh2025multichallenge} increased problem complexity so that advanced models fail to reach above 50\% accuracy on multi-turn tasks that involve instruction following, context assignment, and contextual reasoning. MT-dyna \citep{gao2026mtdyna} presented dynamic questioning techniques based on conversation history and intentions. Multi-turn domain-specific tests include medical consultation evaluators \citep{liao2023medical} and psychological dialogue benchmark sets such as PsycoLLM \citep{hu2025psycollm}. Large-scale simulation tests \citep{laban2025lostinconv} have validated that LLMs get "lost" once they make an initial wrong assumption; Li et al. \citet{li2026consistencylrm}  showed that even reasoning models will disregard true answers when faced with multi-turn challenges. These benchmarks have helped identify failure cases, while SKG-Eval attempts to measure them.
	
	\paragraph{Reference-free and LLM-as-judge evaluation.} A second line replaces n-gram metrics with neural reference-free evaluators. LLM-Eval \citep{lin2023llmeval} unifies multiple dimensions in a single prompt; \citet{zhang2024comprehensive} analyze 30 LLMs across 12 meta-evaluation datasets, exposing both promising correlations with human judgment and significant brittleness under adversarial perturbations. ECoh \citep{mendonca2024ecoh} distills GPT-3.5 into a small open evaluator for turn-level coherence in five languages. BotChat \citep{duan2024botchat} uses an LLM judge to assess pairwise dialogue generation quality. \citet{ike2025humanvsai} compare GPT-4o to human judges across seven KPIs, reporting that GPT-4o handles factuality and commonsense well but struggles with redundancy and self-contradiction---precisely the cross-turn failures SKG-Eval targets. LLM Comparator \citep{kahng2025comparator} addresses interpretability of side-by-side judge outputs but does not change the underlying judging mechanism. Across this literature, the evaluator's state is implicit in the model's attention; whether the judge actually \emph{tracked} the prior turns is unfalsifiable.
	
	\paragraph{Knowledge graphs in dialogue and reasoning evaluation.} Knowledge-graph-grounded dialogue systems are a long-standing line of work, but their use \emph{as evaluators} of free-form multi-turn LLM dialogue is far less developed. The closest benchmarks are surveyed in \citep{li2026survey,guan2026survey}. \citet{yao2026measure} quantify \emph{benchmark} quality (hardness, separability, diversity) but do not address per-conversation quality. To our knowledge, no prior work integrates incremental triple extraction, time-stamped graph state, geometric contradiction detection, and session-level temporal aggregation into a single evaluation framework.
	
	\paragraph{Research gap.} Across (a) benchmarks, which surface failure modes but do not measure them mechanistically; (b) LLM-as-judge protocols, which trust the judge's implicit state and inherit its non-determinism and weak contradiction recall; and (c) reference-based metrics, which score isolated turns---there is no evaluator that (i) maintains an explicit, time-stamped state of the conversation's factual commitments, (ii) detects cross-turn contradictions deterministically and interpretably without an LLM judge in the inner loop, and (iii) aggregates per-turn quality into a session-level metric in a length-invariant way. SKG-Eval is designed to fill exactly this gap.

	\section{Problem Formulation}
	\label{sec:problem}
	
	\paragraph{Dialogue and turn structure.} A dialogue is a sequence of turns $\mathcal{D} = (\tau_1, \tau_2, \ldots, \tau_T)$, where each turn $\tau_t = (q_t, r_t, r^*_t)$ consists of a user prompt $q_t$, an assistant response $r_t$ generated by the system under evaluation, and an optional reference response $r^*_t$ used only when scoring local relevance. We denote the dialogue prefix up to and including turn $t$ by $\mathcal{D}_{1:t}$.
	
	\paragraph{Evaluation as a sequential decision problem.} The goal of an automatic evaluator is to produce, at each turn $t$, a quality score $Q_t \in [0,1]$ together with a session-level summary $\mathcal{S}(\mathcal{D}) \in [0,1]$. We require:
	\begin{itemize}
		\item \textbf{Causality.} $Q_t$ depends only on $\mathcal{D}_{1:t}$. The evaluator may not consult future turns.
		\item \textbf{Statefulness.} $Q_t$ depends on the entire prefix, not only on $\tau_t$. Specifically, the evaluator must be able to detect cross-turn failures in which $r_t$ in isolation appears acceptable but $(r_1,\ldots,r_t)$ is not.
		\item \textbf{Determinism.} Given the same dialogue prefix and the same model parameters (embedding model, extractor), $Q_t$ must be reproducible.
		\item \textbf{Length invariance.} The session-level summary $\mathcal{S}(\mathcal{D})$ should not be artificially inflated or deflated by session length alone.
		\item \textbf{Interpretability.} For any low score, the evaluator should expose the structural cause (which entity, which prior turn, which contradiction class).
	\end{itemize}
	
	\paragraph{Failure modes.}
	We consider six cross-turn conversational failure modes:
	\begin{itemize}
		
		\item \textbf{(F1) Direct contradiction:}
		A response conflicts with a previously asserted fact for the same subject or attribute.
		
		\item \textbf{(F2) Numeric/value substitution:}
		A different numeric or categorical value is assigned to an earlier subject--predicate pair.
		
		\item \textbf{(F3) Antonymic flip:}
		A previously asserted relation is reversed through an antonymic transformation (e.g., ``increases'' vs ``decreases'').
		
		\item \textbf{(F4) Topic drift:}
		New entities or relations are introduced without semantic grounding in prior conversational state.
		
		\item \textbf{(F5) Local irrelevance:}
		The response fails to address the current user query.
		
		\item \textbf{(F6) Silent forgetting:}
		Previously established constraints or commitments are omitted without acknowledgement.
		
	\end{itemize}
	F5 is captured by local relevance; F4 and F6 by historical consistency; F1--F3 by the logical coherence engine. F6 is partially captured by F4 in our framework and fully recovered through structural inspection of the graph.
	
	\paragraph{State representation.} We represent the evaluator's state at turn $t$ by an incremental Semantic Knowledge Graph $G_t = (V_t, E_t)$, defined formally in Section~\ref{sec:method}\cite{Mokg2025}. The graph is a directed multigraph whose nodes are entity labels, whose edges carry typed metadata (relation, attribute, intent, property type, turn id), and whose state at time $t$ summarizes all factual commitments made up to turn $t$ together with semantic-similarity scaffolding induced by an embedding model.
	
	\paragraph{Quality functional.} We posit a per-turn factorization
	\begin{equation}
		Q_t \;=\; \mathcal{F}\!\left(\, S^{\text{loc}}_t,\, S^{\text{cons}}_t,\, S^{\text{log}}_t \,;\, \theta_t \,\right),
		\label{eq:quality_functional}
	\end{equation}
	where $S^{\text{loc}}_t, S^{\text{cons}}_t, S^{\text{log}}_t \in [0,1]$ are the local-relevance, historical-consistency, and logical-coherence scores respectively, and $\theta_t$ is a regime-adaptive weighting selected by simple statistics of $(q_t, r_t)$. The session-level summary is
	\begin{equation}
		\mathcal{S}(\mathcal{D}) \;=\; \mathcal{A}\!\left(\, Q_1,\ldots,Q_T \,;\, \boldsymbol{w}_T \,\right),
		\label{eq:session_aggregator}
	\end{equation}
	where $\mathcal{A}$ is the recency-weighted aggregator of Section~\ref{sec:method:aggregation} and $\boldsymbol{w}_T$ the recency weights. The remainder of the paper develops $\mathcal{F}$, $\mathcal{A}$, and the state $G_t$.

	\begin{table}[t]
		\caption{Key notation used in SKG-Eval.}
		\label{tab:notation}
		\centering
		\small
		\begin{tabular}{@{}ll@{}}
			\toprule
			\textbf{Symbol} & \textbf{Meaning} \\
			\midrule
			$G_t = (V_t, E_t)$ & Semantic  at turn $t$ \\
			$\mathcal{N}_t$ & Nodes introduced at turn $t$ \\
			$\mathcal{C}_t$ & Candidate nodes for contradiction detection \\
			$\phi(\cdot)$ & Sentence embedding function \\
			$\mathrm{cos}(\cdot,\cdot)$ & Cosine similarity \\
			$S^{\text{loc}}_t$ & Local relevance score \\
			$S^{\text{cons}}_t$ & Historical consistency score \\
			$S^{\text{log}}_t$ & Logical coherence score \\
			$Q_t$ & Per-turn quality score \\
			$\mathcal{S}(\mathcal{D})$ & Session-level score \\
			\bottomrule
		\end{tabular}
	\end{table}

	\begin{figure}[t]
		\centering
		\includegraphics[width=\linewidth]{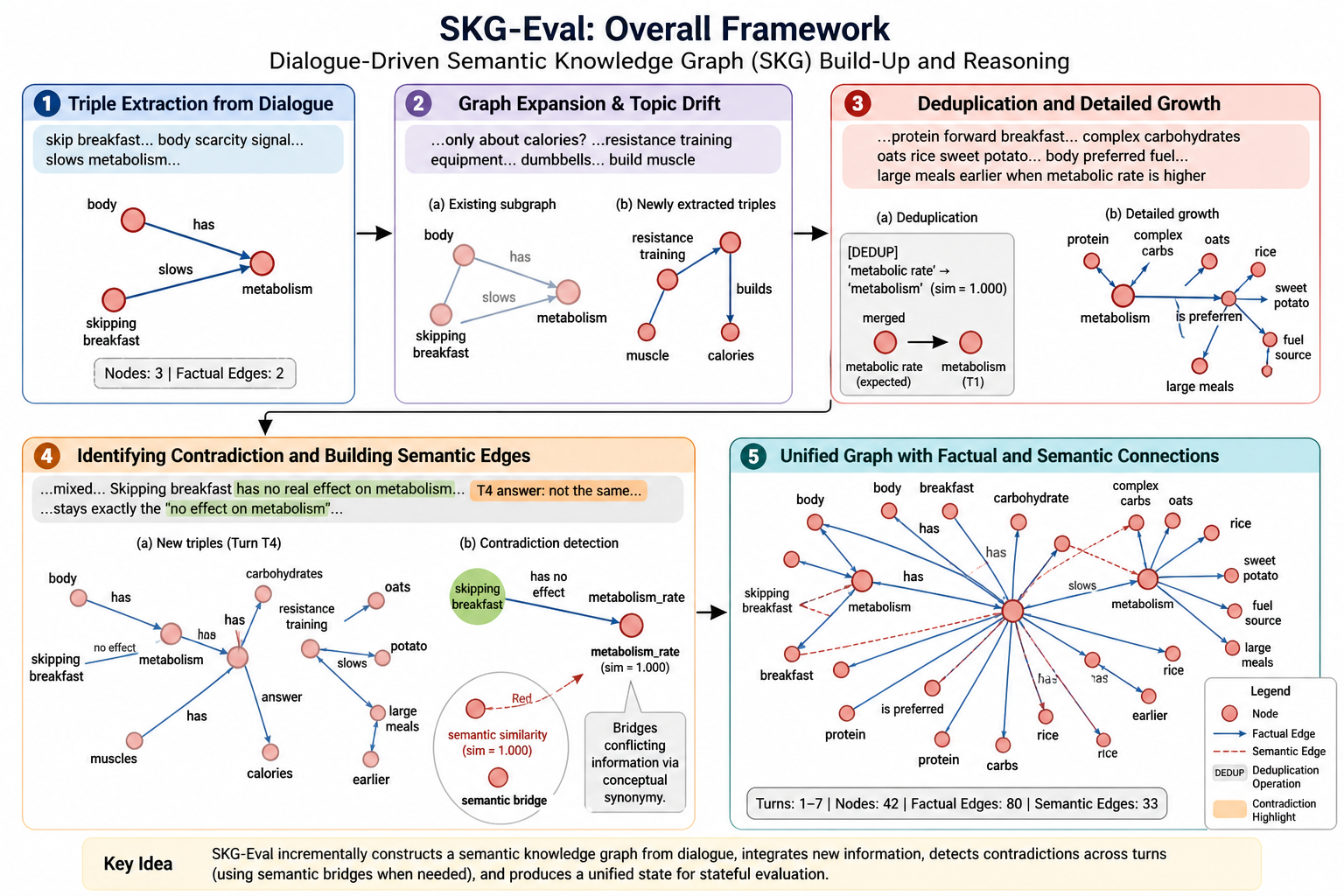}
		\caption{
			End-to-end pipeline of SKG-Eval. Each dialogue turn is converted into structured triples, which are integrated into an evolving Semantic Knowledge Graph (SKG). The framework performs deduplication, semantic linking, and contradiction detection across turns, producing a unified state representation that enables interpretable and state-aware evaluation of multi-turn dialogue.
		}
		\label{fig:overview}
	\end{figure}
	\section{Proposed Method: SKG-Eval}
	\label{sec:method}
	
	As illustrated in Figure~\ref{fig:overview}, SKG-Eval incrementally constructs and reasons over a structured representation of dialogue state. 
	
	We now develop the framework. Section~\ref{sec:method:graph} defines the incremental SKG and its update rule. Sections~\ref{sec:method:slocal}, \ref{sec:method:scons}, and~\ref{sec:method:slogic} develop the three per-turn scores. Section~\ref{sec:method:fusion} defines the regime-adaptive fusion. Section~\ref{sec:method:aggregation} develops the session aggregator. Section~\ref{sec:method:complexity} discusses complexity. Notation: $\mathrm{cos}(\cdot,\cdot)$ denotes cosine similarity; $\phi: \text{strings} \to \mathbb{R}^d$ denotes the sentence-embedding map (a frozen SentenceTransformer in our implementation); $\mathbf{1}[\cdot]$ is the indicator function. We summarize frequently used symbols in Table~\ref{tab:notation}.
	
	\subsection{The Incremental Semantic Knowledge Graph}
	\label{sec:method:graph}
	
	\begin{definition}[Semantic Knowledge Graph]
		\label{def:skg}
		A Semantic Knowledge Graph at time $t$ is a directed multigraph $G_t = (V_t, E_t)$, where each node $v \in V_t$ carries attributes
		\[
		v \;=\; \big(\,\ell(v),\; \tau(v),\; \phi(v),\; \iota(v),\; t_0(v),\; q(v)\,\big),
		\]
		denoting respectively a normalized label, an entity type from a fixed taxonomy $\mathcal{T} = \{\textsc{Person}, \textsc{Event}, \textsc{Object},$ $\textsc{Concept}, \textsc{Condition}, \textsc{Organization}, \textsc{Time}, \textsc{Number}\}$, an embedding $\phi(v) \in \mathbb{R}^d$, an importance score $\iota(v) \in [0,1]$, the introduction turn $t_0(v) \in \mathbb{N}$, and a quarantine flag $q(v) \in \{0,1\}$. Each edge $e \in E_t$ carries
		\[
		e \;=\; \big(\,u, v,\; \rho(e),\; \kappa(e),\; t(e),\; \alpha(e),\; \mu(e),\; \pi(e),\; q(e)\,\big),
		\]
		denoting source/target nodes, relation string $\rho(e)$, edge kind $\kappa(e) \in \{\textsc{fact}, \textsc{semantic}\}$, turn index $t(e)$, attribute $\alpha(e) \in \mathcal{A}$, intent $\mu(e) \in \mathcal{I}$, property type $\pi(e) \in \{\textsc{Exclusive}, \textsc{Additive}\}$, and a quarantine flag $q(e)$.
	\end{definition}
	
	The attribute taxonomy $\mathcal{A}=\{\textsc{def}, \textsc{eff}, \textsc{prop}, \textsc{cmp}, \textsc{req}, \textsc{qty}, \textsc{neg}\}$ encodes \emph{which aspect} of the subject a triple describes (definition, effect, property, comparison, requirement, quantity, negation). The intent taxonomy $\mathcal{I}=\{\textsc{State}, \textsc{Advice}, \textsc{Hyp}\}$ encodes modality. These typed annotations are central to the contradiction engine: only edges with matching attributes and intents are compared.
	
	\paragraph{Triple extraction.} We treat triple extraction as a one-shot LLM call $\mathrm{Extract}: \text{turn text} \to \mathcal{T}_t$, where $\mathcal{T}_t$ is a (possibly empty) set of typed triples
	\[
	T \;=\; (s, r, o,\, \tau_s,\, \tau_o,\, \iota,\, \alpha,\, \mu,\, \pi).
	\]
	The extractor is governed by a deterministic prompt encoding (i) subject normalization rules (strip leading actions, comparative qualifiers, role-of constructs, location qualifiers, possessives), (ii) an attribute taxonomy with disambiguation rules between \textsc{def}, \textsc{eff}, \textsc{prop}, \textsc{qty}, etc., and (iii) a property-type rule that classifies whether multiple values can coexist (\textsc{Additive}) or only one is admissible (\textsc{Exclusive}). The LLM is used \emph{only} for extraction, never for reasoning or judgment, which preserves determinism of the downstream scoring pipeline modulo the extractor.
	
	\paragraph{Update rule.} Given the new triple set $\mathcal{T}_t$ at turn $t$, the graph is updated by Algorithm~\ref{alg:update}. The map $\nu(s) = \texttt{lower}(s).\texttt{strip}()$ is the canonical key. Cross-turn deduplication merges new subjects into existing graph nodes when $\mathrm{cos}(\phi(s), \phi(v)) \ge \theta_{\text{dedup}}$ for some $v \in V_{t-1}$ (with $\theta_{\text{dedup}} = 0.80$ in our implementation), enforcing label consistency across turns---a precondition for cross-turn contradiction detection. After fact edges are added, semantic edges are added between newly introduced nodes and existing nodes whenever their embedding similarity exceeds $\theta_{\text{sem}} = 0.50$.
	
	\subsection{Worked Example: A Three-Turn Walkthrough}
	\label{sec:method:walkthrough}
	We illustrate the incremental graph construction with a three-turn example. 
	Fig.~\ref{fig:skg_growth} shows how new nodes and relations are added at each turn and anchored to the existing graph via factual and semantic connections. 
	
	\paragraph{Conversation.} Three turns with reference answers:
	\begin{itemize}\setlength{\itemsep}{0pt}
		\item[\textbf{T1.}] $q_1$: \emph{``Does skipping breakfast affect metabolism?''} \quad
		$r_1$: \emph{``Yes, skipping breakfast significantly slows down metabolism.''} \quad $r^*_1$: \emph{``Skipping breakfast slows metabolism.''}
		\item[\textbf{T2.}] $q_2$: \emph{``What about concentration?''} \quad
		$r_2$: \emph{``Skipping breakfast also reduces concentration and focus.''} \quad $r^*_2$: \emph{``It reduces concentration.''}
		\item[\textbf{T3.}] $q_3$: \emph{``Does metabolic rate affect weight gain?''} \quad
		$r_3$: \emph{``Yes, a slower metabolic rate is linked to weight gain.''} \quad $r^*_3$: \emph{``Skipping breakfast slows metabolism.''}
	\end{itemize}
	
	\begin{algorithm}[t]
		\caption{\textsc{UpdateGraph}: incremental SKG update at turn $t$}
		\label{alg:update}
		\begin{algorithmic}[1]
			\Require Graph $G_{t-1}$, triple set $\mathcal{T}_t$, embedding model $\phi$, thresholds $\theta_{\text{dedup}}, \theta_{\text{sem}}$
			\State $\mathcal{T}_t \gets \textsc{Deduplicate}(\mathcal{T}_t, V_{t-1}, \phi, \theta_{\text{dedup}})$ \Comment{cross-turn label consistency}
			\State $\mathcal{N}_t \gets \emptyset$ \Comment{new nodes added at turn $t$}
			\For{each triple $T = (s,r,o,\tau_s,\tau_o,\iota,\alpha,\mu,\pi) \in \mathcal{T}_t$}
			\State $u \gets \nu(s),\; v \gets \nu(o)$
			\If{$u \notin V_{t-1}$} add $u$ with $(\ell{=}s, \tau{=}\tau_s, \phi{=}\bot, \iota, t_0{=}t)$; $\mathcal{N}_t \gets \mathcal{N}_t \cup \{u\}$ \EndIf
			\If{$v \notin V_{t-1}$} add $v$ with $(\ell{=}o, \tau{=}\tau_o, \phi{=}\bot, \iota, t_0{=}t)$; $\mathcal{N}_t \gets \mathcal{N}_t \cup \{v\}$ \EndIf
			\State Update existing-node importance: $\iota(u) \gets \max(\iota(u), \iota)$ and similarly for $v$
			\State Add fact edge $u \xrightarrow{r} v$ with $(\kappa{=}\textsc{fact}, t(e){=}t, \alpha, \mu, \pi)$
			\EndFor
			\State \textsc{LinkSemantic}$(G_t, \mathcal{N}_t, \phi, \theta_{\text{sem}})$ \Comment{add cosine-thresholded semantic edges}
			\State \Return $G_t,\, \mathcal{N}_t$
		\end{algorithmic}
	\end{algorithm}
	
	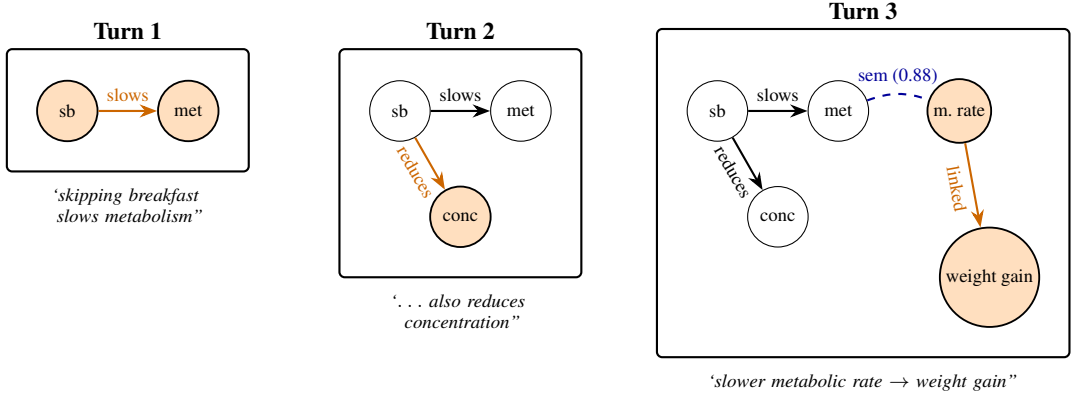
\begin{figure}[t]
		\centering
		\begin{tikzpicture}[
			node/.style={circle, draw, inner sep=0.6mm, font=\scriptsize, minimum size=8mm, fill=white},
			newnode/.style={circle, draw, inner sep=0.6mm, font=\scriptsize, minimum size=8mm, fill=orange!25, line width=0.7pt},
			fact/.style={-Stealth, thick, black},
			fact_new/.style={-Stealth, thick, orange!80!black},
			sem/.style={dashed, thick, blue!60!black},
			panel/.style={draw, thick, rounded corners=2pt, inner sep=4mm},
			lbl/.style={font=\scriptsize, midway, sloped, above},
			ttl/.style={font=\small\bfseries}
			]
			
			\begin{scope}[local bounding box=A]
				\node[newnode] (sb1) at (0,0) {sb};
				\node[newnode] (mt1) at (1.6,0) {met};
				\draw[fact_new] (sb1) -- node[lbl] {slows} (mt1);
			\end{scope}
			\node[ttl, above=4mm of A.north] {Turn 1};
			\node[font=\scriptsize, below=5mm of A.south, text width=32mm, align=center]
			{\emph{`skipping breakfast slows metabolism''}};
			\node[panel, fit=(A)] {};
			
			\begin{scope}[shift={(4.4,0)}, local bounding box=B]
				\node[node] (sb2) at (0,0) {sb};
				\node[node] (mt2) at (1.6,0) {met};
				\node[newnode] (cn2) at (0.8,-1.4) {conc};
				\draw[fact] (sb2) -- node[lbl] {slows} (mt2);
				\draw[fact_new] (sb2) -- node[lbl, sloped, below] {reduces} (cn2);
			\end{scope}
			\node[ttl, above=4mm of B.north] {Turn 2};
			\node[font=\scriptsize, below=5mm of B.south, text width=32mm, align=center]
			{\emph{`$\ldots$ also reduces concentration''}};
			\node[panel, fit=(B)] {};
			
			\begin{scope}[shift={(8.6,0)}, local bounding box=C]
				\node[node] (sb3) at (0,0) {sb};
				\node[node] (mt3) at (1.6,0) {met};
				\node[node] (cn3) at (0.8,-1.4) {conc};
				\node[newnode] (mr3) at (3.2,0.0) {m.\ rate};
				\node[newnode] (wg3) at (3.6,-2.2){weight gain};
				\draw[fact] (sb3) -- node[lbl] {slows} (mt3);
				\draw[fact] (sb3) -- node[lbl, sloped, below] {reduces} (cn3);
				\draw[fact_new] (mr3) -- node[lbl, sloped, below] {linked} (wg3);
				\draw[sem] (mt3) to[bend left=20] node[font=\scriptsize, midway, above] {sem (0.88)} (mr3);
			\end{scope}
			\node[ttl, above=4mm of C.north] {Turn 3};
			\node[font=\scriptsize, below=5mm of C.south, text width=42mm, align=center]
			{\emph{`slower metabolic rate $\to$ weight gain''}};
			\node[panel, fit=(C)] {};
			
		\end{tikzpicture}
		
		\caption{Turn-wise growth of the incremental Semantic Knowledge Graph. Orange nodes and edges denote information introduced at the current turn, black edges denote prior factual commitments, and dashed blue edges denote semantic links induced by embedding similarity. Turn~3 illustrates semantic anchoring of \emph{metabolic rate} to the prior node \emph{metabolism}, followed by factual expansion toward \emph{weight gain}. Abbreviations: \emph{sb}=skipping breakfast, \emph{met}=metabolism, \emph{conc}=concentration.}
		\label{fig:skg_growth}
	\end{figure}
	
	\subsection{Local Relevance via the Semantic Triangle}
	\label{sec:method:slocal}
	
	The local relevance score $S^{\text{loc}}_t$ measures whether $r_t$ addresses $q_t$ and---when a reference $r^*_t$ is available---whether it covers what a correct answer should contain. We compute it as a weighted maximum over sentence-level cosine similarities, with a reference-availability gate that protects against degenerate inputs.
	
	\begin{definition}[Semantic Triangle]
		\label{def:triangle}
		Let $\mathcal{S}(r_t) = \{s_1,\ldots,s_K\}$ be the sentence segmentation of $r_t$, and let $\Phi(\mathcal{S}(r_t)) = [\phi(s_1);\ldots;\phi(s_K)] \in \mathbb{R}^{K \times d}$. Define
		\begin{align}
			m_q &\;=\; \max_{k \le K} \mathrm{cos}\!\big(\phi(q_t),\, \phi(s_k)\big), \\
			m_r &\;=\; \max_{k \le K} \mathrm{cos}\!\big(\phi(r^*_t),\, \phi(s_k)\big) \quad \text{(when } r^*_t \neq \varnothing\text{)}.
		\end{align}
		With reference-availability indicator $\beta_t = \mathbf{1}\!\big[\,r^*_t \neq \varnothing\,\wedge\, |q_t| \ge L_{\min}\,\big]$, the Semantic Triangle local-relevance score is
		\begin{equation}
			S^{\text{loc}}_t \;=\; \begin{cases} \omega_q\, m_q + \omega_r\, m_r & \text{if } \beta_t = 1, \\ m_q & \text{if } \beta_t = 0,\end{cases}
			\quad \text{with } \omega_q + \omega_r = 1.
			\label{eq:triangle}
		\end{equation}
		We use $\omega_q = 0.4,\, \omega_r = 0.6$ and $L_{\min} = 10$ words.
	\end{definition}
	
	\paragraph{Why max-pooling.} The maximum over sentences (rather than the mean) implements a coverage notion: $r_t$ addresses $q_t$ if \emph{some} sentence in $r_t$ is highly aligned with $q_t$. This rewards extended responses that contain a focused answer surrounded by elaboration without diluting the score by averaging in low-relevance sentences.
	
	\paragraph{Reference-availability gate.} The classifier $\beta_t$ guards against two failure modes of naive triangulation. First, when no reference is provided, $m_r$ is undefined; collapsing to $m_q$ keeps the score honest. Second, when the prompt itself is degenerate (e.g.\ a one-word topic label, $|q_t| < L_{\min}$), $m_q$ is dense in the response by construction and the triangle would over-credit reference match alone; collapsing to $m_q$ prevents this.
	
	\paragraph{Short-response context glue.} For \emph{short} responses (word count $< W_{\text{short}}$), pronouns and elliptical references can deflate $m_q$. We apply a context glue: encode the augmented string \texttt{``Prompt: $q_t$ Response: $r_t$''} in lieu of $r_t$ itself for the $S^{\text{loc}}_t$ computation. This is a linguistic anaphora-resolution heuristic, not a semantic change to Definition~\ref{def:triangle}.
	
	\subsection{Historical Consistency via Graph Connectivity and Session Anchor}
	\label{sec:method:scons}
	
	Historical consistency $S^{\text{cons}}_t$ measures whether the new turn's content attaches to the existing conversation. We give two complementary measurements: a structural \emph{graph anchor} score over the typed-edge connectivity of the new nodes, and a sequence-level \emph{session anchor} score that rescues focused-Q\&A patterns where graph disconnection is structurally expected.
	
	\begin{definition}[Graph anchor score]
		\label{def:graph_anchor}
		Let $\mathcal{N}_t$ be the new nodes added at turn $t$, and let $V^{<t} = V_{t-1}$. For each $u \in \mathcal{N}_t$, define the per-node anchor score
		\begin{equation}
			a(u) \;=\; \begin{cases} \eta_{\text{F}} & \text{if } \exists\, v \in V^{<t}\,\text{ with edge } u \to v \text{ or } v \to u \text{ of kind } \textsc{fact}, \\
				\eta_{\text{S}} & \text{else if } \exists\, v \in V^{<t}\,\text{ connected to } u \text{ by a } \textsc{semantic}\text{ edge}, \\
				\eta_{\text{D}} & \text{otherwise (drift)}, \end{cases}
			\label{eq:anchor_perNode}
		\end{equation}
		with $\eta_{\text{F}} = 1.0,\, \eta_{\text{S}} = 0.65,\, \eta_{\text{D}} = 0.20$. The graph anchor score is the importance-weighted mean
		\begin{equation}
			S^{\text{graph}}_t \;=\; \frac{\sum_{u \in \mathcal{N}_t} \iota(u)\, a(u)}{\sum_{u \in \mathcal{N}_t} \iota(u)},
			\label{eq:graph_anchor}
		\end{equation}
		with the convention $S^{\text{graph}}_t = 1$ when $\mathcal{N}_t = \emptyset$ or $V^{<t} = \emptyset$.
	\end{definition}
	
	\begin{remark}[Why factual edges dominate semantic ones]
		A factual edge in $G_t$ is an explicit assertion bound to a specific historical turn---it is the strongest possible attachment of the new node to the dialogue's commitment store. A semantic edge merely indicates embedding proximity to a prior label; this is suggestive but rebuttable. The hierarchy $\eta_{\text{F}} > \eta_{\text{S}} > \eta_{\text{D}}$ encodes this asymmetry. Importance weighting in~\eqref{eq:graph_anchor} prevents incidental low-importance fillers (e.g.\ entity mentions in passing) from dominating the score.
	\end{remark}
	
	\paragraph{Worked example.} Suppose a turn introduces two new nodes, \textit{metabolic rate} and \textit{weight gain}, with importances $0.85$ and $0.60$ respectively. The node \textit{metabolic rate} is connected to the existing node \textit{metabolism} only via a semantic edge (cosine similarity $0.88$ between their label embeddings, above the $\theta_{\text{sem}}{=}0.50$ threshold), so $a(\textit{metabolic rate})=\eta_{\text{S}}=0.65$. The node \textit{weight gain} has no connection to any prior node, so $a(\textit{weight gain})=\eta_{\text{D}}=0.20$. The graph anchor score is then
	\[
	S^{\text{graph}}_t = \frac{0.85 \cdot 0.65 + 0.60 \cdot 0.20}{0.85 + 0.60} = \frac{0.553 + 0.120}{1.45} \approx 0.464,
	\]
	which is the importance-weighted compromise between a moderately-anchored important concept and an isolated minor one. Without importance weighting, both nodes would contribute equally and the score would be the unweighted mean $0.425$, slightly under-crediting the anchored part of the turn.
	
	\paragraph{Session anchor rescue.} Pure graph attachment under-credits Q\&A sessions in which each turn introduces a new sub-topic but all turns share an overarching theme: in such sessions, $\mathcal{N}_t$ is graph-disconnected by design. We rescue the score with a sequence-level anchor.
	
	\begin{definition}[Session anchor]
		Let $\phi^{(1)}$ denote the embedding of the (glued) first turn, fixed at the start of the session. Define the session-anchor similarity
		\begin{equation}
			S^{\text{anc}}_t \;=\; \delta \cdot \mathrm{cos}\!\big(\phi(r_t),\, \phi^{(1)}\big), \qquad \delta \in (0,1].
			\label{eq:session_anchor}
		\end{equation}
		The full historical consistency score is
		\begin{equation}
			S^{\text{cons}}_t \;=\; \max\!\big(\, S^{\text{graph}}_t,\, S^{\text{anc}}_t \,\big).
			\label{eq:s_cons}
		\end{equation}
		We use $\delta = 0.85$, reflecting that the anchor is weaker evidence than a direct graph edge and should not dominate it when the latter is high.
	\end{definition}
	
	\begin{proposition}[Boundedness and monotonicity of $S^{\text{cons}}_t$]
		\label{prop:bound}
		For any $t \ge 1$ and any non-empty $\mathcal{N}_t$, $S^{\text{cons}}_t \in [\eta_{\text{D}},\,1]$. Moreover, if a new turn introduces at least one node $u \in \mathcal{N}_t$ that is fact-connected to $V^{<t}$ and $\sum_{u' \in \mathcal{N}_t} \iota(u') > 0$, then $S^{\text{graph}}_t \ge \eta_{\text{F}} \cdot \min_{u' \in \mathcal{N}_t} \iota(u') / \sum_{u'} \iota(u') > 0$.
	\end{proposition}
	
	The proof is by direct algebra on~\eqref{eq:graph_anchor}; we omit it here.
	
	\subsection{Logical Coherence via the Geometric Contradiction Engine}
	\label{sec:method:slogic}
	
	The logical coherence score $S^{\text{log}}_t$ is the primary component of the framework. 
	The process does not make use of string-level NLI or prompt-based LLM judgment, but instead applies a deterministic geometric reasoner to identify contradictions using an ever-evolving Semantic Knowledge Graph (SKG).
	
	\paragraph{Contradiction as stateful semantic incompatibility.}
	Unlike NLI models applied at the sentence level where contradictions are evaluated on pairs of premises and hypotheses, the SKG-Eval approach views contradiction identification as a state-based incompatibility task involving a semantic memory which is continuously updated during the course of the conversation. This involves building a Semantic Knowledge Graph \cite{Mokg2025} which reflects the set of factual claims made throughout the conversation.
	
	We denote the conversational semantic state at turn $t$ as $\Sigma_t := G_t$,
	where $G_t$ stores all factual and semantic commitments accumulated up to turn $t$.
	
	The contradiction engine uses symbolically derived contradictions combined with geometric compatibility using embeddings. The symbolic approach helps detect high precision logical contradictions, such as negations, antonyms, and numerical contradictions, while geometric similarity provides resistance to paraphrasing and surface variation.

	The contradiction engine checks the current claims for their logical consistency against the past claims tied together through relations and objects similarities to the same semantic anchor, then applies contradiction cascades.
	
	\paragraph{Time-partitioned edge sets.} For each candidate node $u$, define
	\begin{align}
		\mathcal{E}^{\text{cur}}(u) &= \{\,e \in E_t : \kappa(e) = \textsc{fact},\, q(e) = 0,\, t(e) = t,\, e \text{ incident to } u \,\}, \\
		\mathcal{E}^{\text{hist}}(u) &= \{\,e \in E_t : \kappa(e) = \textsc{fact},\, q(e) = 0,\, t(e) < t,\, e \text{ incident to } u \,\}.
	\end{align}
	The geometric engine compares each pair $(e_c, e_h) \in \mathcal{E}^{\text{cur}}(u) \times \widetilde{\mathcal{E}}^{\text{hist}}(u)$ for the candidate set
	\[
	\mathcal{C}_t = \mathcal{N}_t \;\cup\; \big\{\,u \in V_{t-1} : \mathcal{E}^{\text{cur}}(u) \neq \emptyset \,\big\} \;\setminus\; \mathcal{B},
	\]
	where $\mathcal{B}$ is a blocklist of generic pronoun/filler nodes (\texttt{i}, \texttt{you}, \texttt{this}, \texttt{thing}, etc.) that carry no contradiction-bearing semantics.
	
    \paragraph{Revision-aware history filtering.}
    In multi-turn dialogue, users may intentionally revise previously established information (e.g., ``change that to...'', ``instead use...''). Such operations correspond to authorized conversational state updates rather than logical inconsistencies.
    
    We therefore introduce a revision-aware filtering mechanism. Let $\mathcal{R}_t \subseteq E_t$ denote the set of historical edges identified as revision targets at turn $t$. For contradiction evaluation, the historical comparison set is redefined as:
    \begin{equation}
    	\widetilde{\mathcal{E}}^{\text{hist}}(u)
    	\;=\;
    	\mathcal{E}^{\text{hist}}(u) \setminus \mathcal{R}_t.
    \end{equation}
    
    All subsequent contradiction comparisons are then performed on
    \[
    \mathcal{E}^{\text{cur}}(u)
    \times
    \widetilde{\mathcal{E}}^{\text{hist}}(u).
    \]
    
    Importantly, this filtering is \emph{ephemeral}: the underlying graph is not permanently modified, and the suppressed edges remain available for future conversational context.
    
    \paragraph{Intuition through example.}
    Consider a dialogue in which the user intentionally revises a previously established slogan.
    
    \textbf{Turn 1 (initial state).}
    \begin{itemize}
    	\item User: ``I am launching a new coffee brand. The slogan is `Morning Spark: Fueling your day'. What kind of audience does that attract?''
    	
    	\item Assistant: ``That slogan attracts young professionals seeking an energetic start to the day.''
    \end{itemize}
    
    The extracted semantic memory contains:
    \[
    [\texttt{Brand}]
    \rightarrow
    [\texttt{has slogan}]
    \rightarrow
    [\texttt{Morning Spark: Fueling your day}]
    \]
    
    \textbf{Turn 2 (intentional revision).}
    \begin{itemize}
    	\item User: ``Please change the slogan to `Morning Spark: Relax and unwind'. That works better for my decaf line. Who is the audience now?''
    	
    	\item Assistant: ``The revised slogan targets users seeking a calming and soothing morning routine.''
    \end{itemize}
    
    Without revision-aware filtering, the contradiction engine would compare:
    \[
    [\texttt{has slogan}]
    \rightarrow
    [\texttt{Fueling your day}]
    \]
    against
    \[
    [\texttt{has slogan}]
    \rightarrow
    [\texttt{Relax and unwind}]
    \]
    under the same semantic relation. Since the object meanings are directionally opposed, the detector cascade could incorrectly trigger a semantic drift or exclusivity conflict.
    
    However, the phrase:
    \[
    \texttt{``Please change the slogan to...''}
    \]
    acts as an explicit revision signal. Consequently, the earlier slogan edge is temporarily added to $\mathcal{R}_t$ and removed from the contradiction comparison set:
    \[
    \widetilde{\mathcal{E}}^{\text{hist}}(u)
    =
    \mathcal{E}^{\text{hist}}(u)
    \setminus
    \mathcal{R}_t.
    \]
    
    As a result, the new slogan is interpreted as a legitimate conversational update rather than a contradiction. This prevents the evaluator from penalizing dialogue systems for correctly following user-authorized revisions.
    
    The contradiction engine therefore evaluates inconsistencies through progressively weaker regimes:
    (i) explicit polarity contradictions,
    (ii) directional semantic contradictions,
    (iii) structural exclusivity conflicts, and
    (iv) residual semantic divergence.
    This hierarchy ensures that high-confidence symbolic conflicts are resolved before softer embedding-geometric inconsistencies are considered.
	\paragraph{The detector hierarchy.}
	For each pair $(e_c,e_h)$ on a candidate node $u$, the engine computes relation similarity
	\[
	\mathrm{rs}=\mathrm{cos}(\phi(\rho(e_c)),\phi(\rho(e_h)))
	\]
	and object similarity
	\[
	\mathrm{os}=\mathrm{cos}(\phi(\ell(o_c)),\phi(\ell(o_h))),
	\]
	where $o_c$ and $o_h$ are the object endpoints. The pair is then processed by a prioritized cascade of typed contradiction detectors summarized in Table~\ref{tab:detectors}.
	
	The ordering of the cascade reflects contradiction reliability. High-precision symbolic conflicts, such as explicit negation and antonymic reversal, are evaluated first. Abstaining guards (\textsc{IntentGate}, \textsc{ElabGuard}, \textsc{NoiseFloor}, and \textsc{ExclusivityGuard}) suppress categorically incomparable edge pairs and therefore prevent specific classes of false positives. Finally, numeric mismatch, exclusive-object conflict (EOC), and semantic drift are evaluated as progressively softer forms of incompatibility.
	
	The cascade exits at the first detector that fires, ensuring that strong symbolic contradictions dominate weaker embedding-geometric inconsistencies. The field $\rho_{\text{type}}$ denotes the relation-type label assigned by the extractor (cf. Section~\ref{sec:method:graph}); $\mathcal{A}_{\text{ant}}$ is a curated antonym dictionary (\textit{increases}/\textit{decreases}, \textit{causes}/\textit{prevents}, \textit{always}/\textit{never}, etc.); and $\mathcal{M}_{\neg}$ is a fixed set of negation markers. Same-type groups (\textsc{Person}, \textsc{Number}, $\ldots$) are defined according to the entity taxonomy in Definition~\ref{def:skg}.
	
	\paragraph{Revision-aware guard.}
	The revision filtering described above is applied prior to the detector cascade and acts as a higher-priority guard. Consequently, even when a pair $(e_c,e_h)$ satisfies all geometric and type constraints, it is excluded from contradiction evaluation if $e_h \in \mathcal{R}_t$. This ensures that the engine distinguishes between model inconsistency and user-directed state updates.
	
	\paragraph{Unified contradiction operator.}
	We define contradiction confidence as:
	\begin{equation}
		c(e_c,e_h) = \max_{k \in \mathcal{D}} c_k(e_c,e_h),
	\end{equation}
	where $\mathcal{D}$ denotes the set of geometric detectors.
	
\begin{table}[t]
	\caption{
		Geometric contradiction detector cascade.
		Detectors are evaluated top-to-bottom, and the first firing detector assigns the contradiction confidence used in $S^{\text{log}}_t$.
		Italicized rows denote abstaining guards that suppress categorically incomparable comparisons.
		Implementation thresholds are reported in Appendix~\ref{app:parameters}.
		Revision-filtered edges (cf.~$\widetilde{\mathcal{E}}^{\text{hist}}(u)$) are excluded prior to contradiction evaluation.
	}
	\label{tab:detectors}
	\centering
	\small
	\begin{tabular}{@{}p{2.9cm}p{6.2cm}p{2.6cm}@{}}
		\toprule
		\textbf{Detector} & \textbf{Firing condition} & \textbf{Confidence $c$} \\
		\midrule
		
		\textsc{NegFlip}
		&
		exactly one of $\rho(e_c),\rho(e_h)\in\mathcal{M}_{\neg}$;
		$\mathrm{os}>\theta^{\text{neg}}_{\text{obj}}$;
		both $\rho_{\text{type}}\notin\{\textsc{elab},\textsc{sol},\textsc{diag}\}$
		&
		$0.95$
		\\
		
		\textsc{Antonym}
		&
		$(\rho(e_c),\rho(e_h))$ opposite in $\mathcal{A}_{\text{ant}}$,
		$\mathrm{os}>\theta^{\text{obj}}_{\min}$
		&
		$0.88$
		\\
		
		\midrule
		
		\textit{\textsc{IntentGate}}
		&
		\textit{$\mu(e_c)\neq\mu(e_h)$}
		&
		\textit{abstain}
		\\
		
		\textit{\textsc{ElabGuard}}
		&
		\textit{$\rho_{\text{type}}(e_c)\in\{\textsc{elab},\textsc{sol},\textsc{diag}\}$}
		&
		\textit{abstain}
		\\
		
		\textit{\textsc{NoiseFloor}}
		&
		\textit{$\mathrm{rs}<\theta^{\text{rel}}_{\min}$ or $\mathrm{os}<\theta^{\text{obj}}_{\min}$}
		&
		\textit{abstain}
		\\
		
		\midrule
		
		\textsc{NumMismatch}
		&
		$\mathrm{rs}>0.70$ and extracted numerals satisfy $n_c\neq n_h$
		&
		$0.92$
		\\
		
		\textsc{EOC}
		&
		$\mathrm{rs}>0.85$,
		$\mathrm{os}<\theta^{\text{obj}}_{\text{div}}$,
		both predicates satisfy $\pi=\textsc{Excl}$
		&
		$1-\mathrm{os}$
		\\
		
		\textsc{Same-Type EOC}
		&
		$\mathrm{rs}>0.85$,
		same-type entity group,
		$\mathrm{os}<\theta^{\text{ST}}$,
		both predicates satisfy $\pi=\textsc{Excl}$
		&
		$\max(0.60,\,1-\mathrm{os})$
		\\
		
		\textsc{Residual Semantic Drift}
		&
		$\mathrm{rs}>0.60$,
		$\theta^{\text{obj}}_{\min}<\mathrm{os}<0.75$
		&
		$\begin{cases}
			1-\mathrm{os}, & \mathrm{os}<0.40\\
			0.45, & 0.40\le \mathrm{os}<0.75
		\end{cases}$
		\\
		
		\bottomrule
	\end{tabular}
\end{table}
	
	\begin{figure}[t]
		\centering
		\begin{tikzpicture}[
			node distance=3.2mm and 6mm,
			start/.style={draw, rounded corners=4pt, fill=blue!8, font=\scriptsize, minimum width=32mm, align=center, inner sep=1.5mm},
			det/.style={draw, fill=orange!10, font=\scriptsize, minimum width=32mm, align=center, inner sep=1.2mm},
			guard/.style={draw, fill=gray!12, font=\scriptsize\itshape, minimum width=32mm, align=center, inner sep=1.2mm},
			finalbox/.style={draw, rounded corners=4pt, fill=red!12, font=\scriptsize, minimum width=26mm, align=center, inner sep=1.2mm},
			absent/.style={draw, dashed, fill=gray!5, font=\scriptsize\itshape, minimum width=26mm, align=center, inner sep=1.2mm},
			yes/.style={-Stealth, thick, red!70!black, font=\scriptsize},
			no/.style={-Stealth, thick, black!60, font=\scriptsize},
			done/.style={-Stealth, thick, black, font=\scriptsize}
			]
			
			\node[start] (s) {
				pair $(e_c, e_h)$ on candidate $u$\\
				$(e_c,e_h)\in \mathcal{E}^{\text{cur}}(u)\times \widetilde{\mathcal{E}}^{\text{hist}}(u)$\\
				compute $\mathrm{rs}, \mathrm{os}$
			};
			
			\node[guard, below=of s] (rf) {\textsc{RevisionFilter}\\$e_h \notin \mathcal{R}_t$};
			
			\node[det, below=of rf] (d1) {\textsc{NegFlip}\\$\mathrm{os}>0.40$};
			\node[det, below=of d1] (d2) {\textsc{Antonym}\\$\mathrm{os}>\theta^{\text{obj}}_{\min}$};
			
			\node[guard, below=of d2] (g1) {\textsc{IntentGate} / \textsc{ElabGuard} /\\ \textsc{NoiseFloor}};
			
			\node[det, below=of g1] (d3) {\textsc{NumMismatch}\\$\mathrm{rs}>0.70,\;n_c{\neq}n_h$};
			\node[det, below=of d3] (d4) {\textsc{EOC} (Std/Same-type)\\$\mathrm{rs}>0.85,\;\pi=\textsc{Excl}$};
			\node[det, below=of d4] (d5) {\textsc{SemDrift} (Strong/Mod)\\$\mathrm{rs}>0.60$};
			
			\node[absent, below=of d5] (none) {no fire $\Rightarrow c=0$};
			
			\node[finalbox, right=24mm of d1] (o1) {$c=0.95$};
			\node[finalbox, right=24mm of d2] (o2) {$c=0.88$};
			\node[absent, right=24mm of g1] (og) {abstain ($c=0$)};
			\node[finalbox, right=24mm of d3] (o3) {$c=0.92$};
			\node[finalbox, right=24mm of d4] (o4) {$1{-}\mathrm{os}$ or $\max(0.6,1{-}\mathrm{os})$};
			\node[finalbox, right=24mm of d5] (o5) {$1{-}\mathrm{os}$ or $0.45$};
			
			\draw[done] (s) -- (rf);
			\draw[no] (rf) -- node[right] {pass} (d1);
			
			\foreach \a/\b in {d1/d2, d2/g1, g1/d3, d3/d4, d4/d5, d5/none} {
				\draw[no] (\a) -- node[right] {no} (\b);
			}
			
			\draw[yes] (d1.east) -- node[above] {yes} (o1.west);
			\draw[yes] (d2.east) -- node[above] {yes} (o2.west);
			\draw[yes] (g1.east) -- node[above] {fires} (og.west);
			\draw[yes] (d3.east) -- node[above] {yes} (o3.west);
			\draw[yes] (d4.east) -- node[above] {yes} (o4.west);
			\draw[yes] (d5.east) -- node[above] {yes} (o5.west);
			
		\end{tikzpicture}
		
		\caption{
			Geometric contradiction detector cascade with revision-aware filtering. 
			Each pair $(e_c, e_h)$ is drawn from the comparison set 
			$\mathcal{E}^{\text{cur}}(u)\times \widetilde{\mathcal{E}}^{\text{hist}}(u)$, 
			where $\widetilde{\mathcal{E}}^{\text{hist}}(u)$ excludes revision-targeted edges 
			$\mathcal{R}_t$. The \textsc{RevisionFilter} acts as a pre-cascade guard, removing 
			user-intended updates before contradiction evaluation. The remaining pairs are 
			processed top-to-bottom; the first detector that fires assigns confidence $c$. 
			Italic gray nodes denote abstaining guards that suppress incomparable pairs. 
			The final score is 
			$S^{\text{log}}_t = 1 - \max_{u \in \mathcal{C}_t} \max_{(e_c,e_h)} c$.
		}
		
		\label{fig:cascade}
	\end{figure}
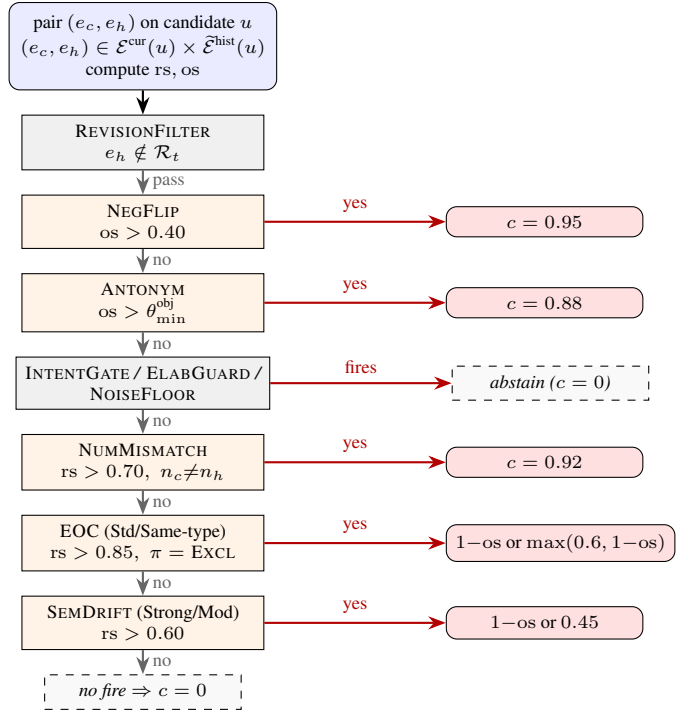
	
	\paragraph{Score aggregation.} For each candidate node $u$, let $c^*(u) = \max_{(e_c, e_h)} c(e_c, e_h)$ be the maximum confidence over all detector firings ($0$ if none). The logical coherence score is
	\begin{equation}
		S^{\text{log}}_t \;=\; 1 \;-\; \max_{u \in \mathcal{C}_t}\, c^*(u),
		\label{eq:slogic}
	\end{equation}
	with $S^{\text{log}}_t = 1$ when no detector fires.
	
	\paragraph{Why geometric rather than NLI.}
	String-level NLI cross-encoders process linearized premise--hypothesis pairs and exhibit three failure modes that occur frequently in long dialogue: 
	(i) numeric substitution within otherwise identical claims, 
	(ii) antonymic or paraphrased contradiction across surface-different statements, and 
	(iii) contradictions whose conflicting evidence is buried in long conversational prefixes. 
	
	SKG-Eval addresses these limitations by operating directly on structured semantic state rather than serialized text. Contradiction detection is performed over typed graph edges with explicit symbolic reasoning, embedding-geometric compatibility analysis, and revision-aware filtering that suppresses user-intended updates prior to contradiction evaluation.
	
	\begin{proposition}[Conditions favoring geometric contradiction reasoning]
		\label{prop:dominance}
		Let $f_{\textsc{NLI}}$ denote a string-level NLI cross-encoder with bounded effective context $L_{\max}$, and let $f_{\textsc{GEO}}$ denote the proposed geometric contradiction engine. The geometric engine is expected to exhibit higher contradiction-recall than $f_{\textsc{NLI}}$ under the following conditions:
		\begin{itemize}
			\item[(i)] \textbf{Numeric substitution.} 
			Contradictory claims differ primarily in symbolic values embedded within otherwise similar contexts (e.g.\ ``boils at 100$^\circ$C'' vs.\ ``boils at 90$^\circ$C'').
			
			\item[(ii)] \textbf{Long-prefix contradiction.} 
			The contradictory historical claim lies outside the effective context window of the NLI encoder.
			
			\item[(iii)] \textbf{Antonymic paraphrase.} 
			The contradiction is expressed through semantically opposing relations under paraphrased surface forms (e.g.\ ``increases'' vs.\ ``reduces'').
		\end{itemize}
	\end{proposition}
	
	\begin{proof}[Proof sketch]
		Numeric substitutions are difficult for text-level semantic models because shared contextual tokens dominate the representation, whereas the proposed engine isolates and compares symbolic values directly. Long-prefix contradictions may be truncated or attenuated in string-level NLI, while SKG-Eval retrieves historical claims through graph-indexed semantic anchors independent of dialogue length. Antonymic paraphrases are handled explicitly through relation-level opposition rather than implicit encoder generalization. Revision-aware filtering further distinguishes user-directed state updates from genuine model inconsistency.
	\end{proof}
	
	The proposition does not claim universal dominance over NLI systems. Rather, it identifies the specific contradiction regimes that dominate long-form conversational inconsistency and for which structured geometric reasoning is particularly effective.
	
	\begin{proposition}[Determinism and complexity of the engine]
		\label{prop:slogic}
		Given a fixed embedding model $\phi$, fixed extractor outputs $\mathcal{T}_{1:t}$, and deterministic revision filtering, the score $S^{\text{log}}_t$ is deterministic. The time complexity at turn $t$ is
		\[
		\mathcal{O}\!\left(|\mathcal{C}_t| \cdot \bar{C}\bar{H}\right),
		\]
		where $\bar{C}=\mathbb{E}[|\mathcal{E}^{\text{cur}}(u)|]$ and $\bar{H}=\mathbb{E}[|\widetilde{\mathcal{E}}^{\text{hist}}(u)|]$. Since $\bar{C}$ and $\bar{H}$ remain small in practice, the effective complexity is near-linear in the number of turns.
	\end{proposition}
	
	The complexity follows from bounded per-node edge comparisons. Embeddings are computed once per edge, and each detector executes in constant time. The resulting computation is dominated by embedding evaluation and is naturally parallelizable across candidate nodes.
	
	\paragraph{Quarantine.} 
	When $Q_t < \theta_{\text{quar}} = 0.40$, all nodes and edges introduced at turn $t$ are marked $q(\cdot) = 1$ and excluded from subsequent contradiction checks and consistency scoring. Quarantine prevents low-quality content from propagating through the graph state and serves as the framework's analogue to a hypothesis-rejection mechanism. Edges suppressed via revision filtering (i.e., those in $\mathcal{R}_t$) are not considered erroneous and therefore do not trigger quarantine.
	
	\subsection{Regime-Adaptive Fusion}
	\label{sec:method:fusion}
	
	The three scores are fused by a regime-adaptive convex combination,
	\begin{equation}
		\bar{Q}_t \;=\; w^{\text{loc}}_t \cdot S^{\text{loc}}_t \;+\; w^{\text{cons}}_t \cdot S^{\text{cons}}_t \;+\; w^{\text{log}}_t \cdot S^{\text{log}}_t, \quad w^{\text{loc}}_t + w^{\text{cons}}_t + w^{\text{log}}_t = 1,
		\label{eq:fusion}
	\end{equation}
	where the weights depend on a turn-level regime selector. Define the regime indicator
	\[
	g_t \;=\; 
	\begin{cases} 
		\textsc{Short} & |r_t| < W_{\text{short}}, \\ 
		\textsc{QA} & |r_t| \ge W_{\text{short}}\,\wedge\,|q_t| < W_{\text{qa}}, \\ 
		\textsc{General} & \text{otherwise.} 
	\end{cases}
	\]
	
	The weight profile is selected by lookup, $\theta_t = \Theta[g_t]$, with the three profiles
	\[
	\Theta[\textsc{Short}] = (0.50,\, 0.10,\, 0.40), \qquad
	\Theta[\textsc{QA}] = (0.65,\, 0.05,\, 0.30), \qquad
	\Theta[\textsc{General}] = (0.50,\, 0.20,\, 0.30).
	\]
	
	\textsc{Short} responses (e.g.\ ``Yes.'', ``42.'') legitimately produce low $S^{\text{cons}}_t$ because they introduce few new nodes; the profile down-weights consistency. \textsc{QA} sessions are encyclopedic in nature; consistency is down-weighted further while local relevance is up-weighted, since each turn is a largely self-contained sub-question. \textsc{General} dialogue is the default mixed regime.
	
	These guards ensure that strong logical failures dominate scoring, while preventing high logical coherence from masking relevance or consistency failures.
	
	\paragraph{Guard cascades.} Three monotone guards refine the convex combination:
	
	\begin{enumerate}
		\item \textbf{Hard logic gate.} If $S^{\text{log}}_t < \theta^{\text{log}}_{\text{hard}} = 0.60$, set $\bar{Q}_t \gets \min(\bar{Q}_t,\, 0.40)$. A confirmed contradiction firmly fails the turn regardless of high local relevance. (Revision-suppressed comparisons do not contribute to $S^{\text{log}}_t$ and therefore do not trigger this gate.)
		
		\item \textbf{Joint weakness penalty.} If $S^{\text{loc}}_t < 0.50$ \emph{and} $S^{\text{cons}}_t < 0.45$ \emph{and} $g_t \neq \textsc{Both-Short}$, set $\bar{Q}_t \gets \mu_{\text{joint}} \bar{Q}_t$ with $\mu_{\text{joint}} = 0.75$. This prevents $S^{\text{log}}_t = 1$ from compensating for responses that are simultaneously off-topic and disconnected.
		
		\item \textbf{Non-sequitur softening.} If $S^{\text{cons}}_t < 0.45$ \emph{and} $S^{\text{loc}}_t < 0.20$ \emph{and} $g_t \neq \textsc{Both-Short}$, set $\bar{Q}_t \gets 0.5\, \bar{Q}_t$.
	\end{enumerate}
	
	The final per-turn quality is $Q_t = \mathrm{clip}(\bar{Q}_t,\, 0,\, 1)$. All guards are monotone non-increasing in $\bar{Q}_t$; they cannot inflate the score, only reduce it under structural failure conditions.
	
	\paragraph{Reporting thresholds.} 
	For all empirical analyses we treat $Q_t \ge \theta_{\text{pass}} = 0.60$ as a passing turn and $S^{\text{loc}}_t < 0.50$ as a hard relevance failure (the turn is reported as failed regardless of $Q_t$, complementing the hard-logic gate above). These thresholds are used solely for reporting and evaluation purposes and are not part of the scoring functional that produces $Q_t$. 
	
	Revision-suppressed comparisons do not affect these thresholds, as they are excluded prior to the computation of $S^{\text{log}}_t$ and therefore do not influence the pass/fail decision.
	
	\begin{proposition}[Threshold invariance of ranking]
		\label{prop:threshold_invariance}
		Let $Q_t \in [0,1]$ denote the continuous turn-level score produced by the SKG-Eval scoring functional, and let $\theta_{\text{pass}}$ be a reporting threshold used only to assign binary pass/fail labels. For any two turns $i$ and $j$, if $Q_i > Q_j$, then their relative ranking remains unchanged for any choice of $\theta_{\text{pass}}$.
	\end{proposition}
	
	\begin{proof}
		The threshold $\theta_{\text{pass}}$ is applied only after $Q_t$ has been computed and maps each score to a reporting label $\mathbf{1}[Q_t \ge \theta_{\text{pass}}]$. Since the thresholding operation does not modify the underlying continuous scores $Q_i$ and $Q_j$, the ordering induced by $Q_i > Q_j$ is invariant to the choice of $\theta_{\text{pass}}$.
	\end{proof}
	
	\subsection{Session-Level Aggregation: Recency-Weighted Trend}
	\label{sec:method:aggregation}
	
	A session-level summary that simply averages per-turn quality is biased against improving sessions and overly lenient toward degrading ones. We therefore aggregate via a recency-weighted regression with a length-adaptive trend coefficient.
	
	\begin{definition}[Recency weights]
		For a session of length $T$ and decay rate $\gamma > 0$, define
		\begin{equation}
			w_i \;=\; \frac{e^{\gamma\, (i-1)}}{\sum_{j=1}^{T} e^{\gamma\, (j-1)}},\qquad i = 1,\ldots, T.
			\label{eq:recency_weights}
		\end{equation}
	\end{definition}
	
	\begin{definition}[Session aggregator]
		\label{def:agg}
		Let $\hat{Q} = (Q_1, \ldots, Q_T)$ denote the sequence of turn-level scores and $\boldsymbol{w} = (w_1,\ldots,w_T)$ the corresponding recency weights. Compute
		\begin{align}
			\bar{Q}^{\text{rec}} &\;=\; \sum_{i=1}^T w_i\, Q_i, \label{eq:layerA}\\
			(\hat{\beta},\, \hat{\alpha}) &\;=\; \arg\min_{\beta,\alpha}\, \sum_{i=1}^T w_i \,(Q_i - \alpha - \beta\,(i-1))^2, \label{eq:layerB}\\
			\lambda_{\text{eff}} &\;=\; \lambda_{\text{base}}\,\frac{T}{T_{\text{ref}}}, \label{eq:adaptive_lambda}\\
			\mathcal{S}(\mathcal{D}) &\;=\; \mathrm{clip}\!\big(\, \bar{Q}^{\text{rec}} + \lambda_{\text{eff}}\,\hat{\beta},\; 0,\, 1\,\big).
			\label{eq:final_score}
		\end{align}
		We use $\gamma = 0.1$, $\lambda_{\text{base}} = 5.0$, and $T_{\text{ref}} = 20$.
	\end{definition}
	
	\paragraph{Why two layers.} 
	\eqref{eq:layerA} captures the \emph{level} (current operating quality), giving more weight to recent turns. 
	\eqref{eq:layerB} captures the \emph{slope} (whether the conversation is improving or degrading). 
	Their combination via~\eqref{eq:final_score} integrates both instantaneous quality and temporal trend, 
	with the contribution of the trend term scaled by $\lambda_{\text{eff}}$.
	
	\paragraph{Why adaptive $\lambda$.} 
	A fixed $\lambda$ would cause a slope of fixed magnitude to contribute the same adjustment irrespective of session length $T$. 
	However, a slope $\beta$ sustained over $T$ turns induces a cumulative level change proportional to $\beta T$. 
	Accordingly, a length-aware coefficient $\lambda_{\text{eff}} \propto T$ ensures that the contribution of the trend term is properly normalized across sessions of different lengths. 
	Calibrating to $T_{\text{ref}} = 20$ preserves consistency with the canonical choice $\lambda = 5$ on reference-length sessions.
	
	\begin{proposition}[Shift invariance and slope unbiasedness]
		\label{prop:agg}
		The aggregator $\mathcal{S}$ in Definition~\ref{def:agg} satisfies:
		(i) \emph{Shift invariance:} for any constant $\Delta$ such that $Q_i + \Delta \in [0,1]$ for all $i$, we have $\mathcal{S}(\hat{Q} + \Delta) = \mathcal{S}(\hat{Q}) + \Delta$ prior to clipping. 
		(ii) \emph{Slope unbiasedness:} the estimator $\hat{\beta}$ obtained in~\eqref{eq:layerB} is the weighted least-squares (WLS) slope, which is unbiased under the model $Q_i = \alpha + \beta(i-1) + \varepsilon_i$, where $\mathbb{E}[\varepsilon_i]=0$ and $\mathrm{Var}(\varepsilon_i)<\infty$.
	\end{proposition}
	
	Property (i) follows because both $\bar{Q}^{\text{rec}}$ and the fitted intercept $\hat{\alpha}$ shift by $\Delta$, while the slope $\hat{\beta}$ remains unchanged. 
	Property (ii) is the standard unbiasedness result for weighted least-squares estimators under fixed weights.
	
	\subsection{Putting It All Together: Complexity and Determinism}
	\label{sec:method:complexity}
	
	The full per-turn pipeline at turn $t$ is given in Algorithm~\ref{alg:scoreTurn}. The per-turn cost is dominated by:
	(a) one extractor LLM call on $r_t$, 
	(b) $\mathcal{O}(|\mathcal{T}_t|)$ graph updates, 
	(c) $\mathcal{O}(|\mathcal{N}_t| \cdot |V_{t-1}|)$ embedding cosine computations for semantic linking and consistency scoring, and 
	(d) $\mathcal{O}(|\mathcal{C}_t|\cdot \bar{C}\bar{H})$ detector evaluations within the geometric engine.
	
	Given a fixed extractor output $\mathcal{T}_{1:t}$, fixed embedding model $\phi$, and deterministic revision-filtering mechanism, all components except (a) are deterministic. In practice, $\bar{C}$ and $\bar{H}$ remain small due to bounded per-node edge counts, making the effective complexity near-linear in the number of turns.
	
	The session-level aggregation is performed once per session in $\mathcal{O}(T)$ time. This makes SKG-Eval scalable to long conversations where repeated LLM-based judging over growing dialogue prefixes becomes increasingly expensive.
	
\begin{algorithm}[t]
	\caption{\textsc{ScoreTurn}: per-turn evaluation in SKG-Eval}
	\label{alg:scoreTurn}
	\begin{algorithmic}[1]
		\Require Prompt $q_t$, response $r_t$, optional reference $r^*_t$, prior graph $G_{t-1}$, embedding model $\phi$
		\State Classify response/prompt regime $g_t$; select weights $\theta_t = \Theta[g_t]$
		\State Compute $S^{\text{loc}}_t$ via Definition~\ref{def:triangle} (with context glue if $g_t = \textsc{Short}$)
		\State $\mathcal{T}_t \gets \textsc{Extract}(\text{turn text})$
		\State $(G_t, \mathcal{N}_t) \gets \textsc{UpdateGraph}(G_{t-1}, \mathcal{T}_t, \phi, \theta_{\text{dedup}}, \theta_{\text{sem}})$
		\State Compute $S^{\text{graph}}_t$ via~\eqref{eq:graph_anchor}; compute $S^{\text{anc}}_t$ via~\eqref{eq:session_anchor}
		\State $S^{\text{cons}}_t \gets \max(S^{\text{graph}}_t, S^{\text{anc}}_t)$
		\State $b_t \gets \textsc{IsRevisionPrompt}(q_t)$
		\State $\mathcal{R}_t \gets \textsc{ExtractRevisionTargets}(q_t, G_{t-1})$ if $b_t=1$, else $\emptyset$
		\State $S^{\text{log}}_t \gets$ \textsc{Geometric-Engine}$(G_t, \mathcal{N}_t, \mathcal{T}_t, \phi, \mathcal{R}_t)$ via~\eqref{eq:slogic}
		\State $\bar{Q}_t \gets \theta_t^{\top}(S^{\text{loc}}_t, S^{\text{cons}}_t, S^{\text{log}}_t)$
		\State Apply hard-logic gate, joint-weakness penalty, non-sequitur softening
		\State $Q_t \gets \mathrm{clip}(\bar{Q}_t, 0, 1)$; if $Q_t < \theta_{\text{quar}}$ quarantine turn-$t$ nodes/edges in $G_t$
		\State \Return $Q_t,\, S^{\text{loc}}_t,\, S^{\text{cons}}_t,\, S^{\text{log}}_t,\, G_t$
	\end{algorithmic}
\end{algorithm}
	
	\paragraph{Determinism and reproducibility.} 
	Modulo the (off-line, low-temperature) extractor, every score in SKG-Eval is a deterministic function of the inputs, the embedding model, and a small set of fixed thresholds, including the revision-filtering mechanism. This stands in contrast to LLM-as-judge protocols, whose outputs may vary across decoding seeds and prompt orderings, and whose run-to-run variance can be comparable to the inter-method differences they aim to measure.
	
	\paragraph{Interpretability.} 
	For every low score, SKG-Eval surfaces the exact structural cause: the disconnected nodes that reduce $S^{\text{cons}}_t$, the (current edge, historical edge) pair that triggers a contradiction along with the detector type and confidence, and the regime that determines the weighting. Each contradiction certificate is a tuple $(u,\, e_c,\, e_h,\, \text{detector},\, c)$, which can be presented to a human auditor or used in dataset construction for negative example mining. We argue that such an explicit audit trail is a necessary condition for evaluator trust at the session level—a property that judge-LLM protocols typically cannot provide without additional external mechanisms.
	
	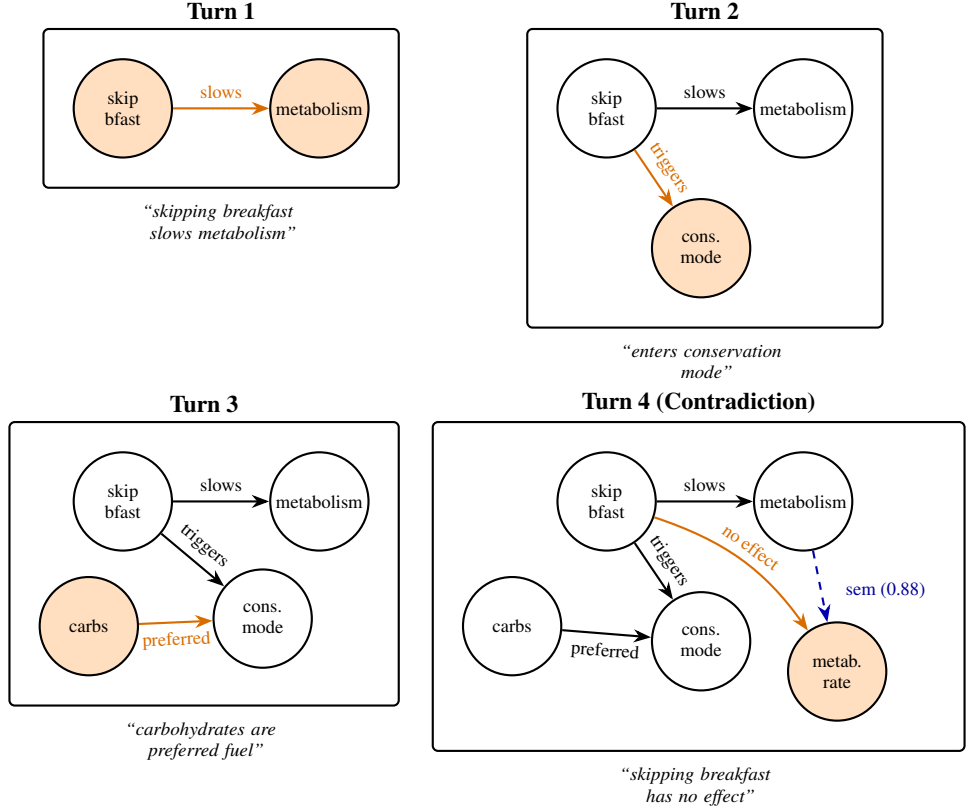
\begin{figure}[t]
		\centering
		\begin{tikzpicture}[
			node/.style={circle, draw, thick, fill=white, minimum size=13mm, inner sep=0.5mm, font=\scriptsize, align=center},
			newnode/.style={circle, draw, thick, fill=orange!25, minimum size=13mm, inner sep=0.5mm, font=\scriptsize, align=center},
			fact/.style={-Stealth, thick, black},
			factnew/.style={-Stealth, thick, orange!85!black},
			sem/.style={dashed, -Stealth, thick, blue!70!black},
			panel/.style={draw, thick, rounded corners=2pt, inner sep=4mm},
			ttl/.style={font=\small\bfseries},
			lbl/.style={font=\scriptsize, midway, above, sloped},
			quote/.style={font=\scriptsize\itshape, align=center, text width=38mm}
			]
			
			\begin{scope}[local bounding box=A]
				\node[newnode] (sb1) at (0,0) {skip\\bfast};
				\node[newnode] (met1) at (2.6,0) {metabolism};
				\draw[factnew] (sb1) -- node[lbl] {slows} (met1);
			\end{scope}
			\node[ttl, above=4mm of A.north] {Turn 1};
			\node[quote, below=5mm of A.south] {``skipping breakfast\\slows metabolism''};
			\node[panel, fit=(A)] {};
			
			\begin{scope}[shift={(6.4,0)}, local bounding box=B]
				\node[node] (sb2) at (0,0) {skip\\bfast};
				\node[node] (met2) at (2.6,0) {metabolism};
				\node[newnode] (cm2) at (1.25,-1.85) {cons.\\mode};
				
				\draw[fact] (sb2) -- node[lbl] {slows} (met2);
				\draw[factnew] (sb2) -- node[lbl] {triggers} (cm2);
			\end{scope}
			\node[ttl, above=4mm of B.north] {Turn 2};
			\node[quote, below=5mm of B.south] {``enters conservation\\mode''};
			\node[panel, fit=(B)] {};
			
			\begin{scope}[shift={(0,-5.2)}, local bounding box=C]
				\node[node] (sb3) at (0,0) {skip\\bfast};
				\node[node] (met3) at (2.6,0) {metabolism};
				\node[newnode] (carb3) at (-0.45,-1.65) {carbs};
				\node[node] (cm3) at (1.85,-1.55) {cons.\\mode};
				
				\draw[fact] (sb3) -- node[lbl] {slows} (met3);
				\draw[fact] (sb3) -- node[lbl] {triggers} (cm3);
				\draw[factnew] (carb3) -- node[lbl, below] {preferred} (cm3);
			\end{scope}
			\node[ttl, above=4mm of C.north] {Turn 3};
			\node[quote, below=5mm of C.south] {``carbohydrates are\\preferred fuel''};
			\node[panel, fit=(C)] {};
			
			\begin{scope}[shift={(6.4,-5.2)}, local bounding box=D]
				\node[node] (sb4) at (0,0) {skip\\bfast};
				\node[node] (met4) at (2.6,0) {metabolism};
				\node[node] (carb4) at (-1.25,-1.65) {carbs};
				\node[node] (cm4) at (1.25,-1.85) {cons.\\mode};
				\node[newnode] (mr4) at (3.05,-2.25) {metab.\\rate};
				
				\draw[fact] (sb4) -- node[lbl] {slows} (met4);
				\draw[fact] (sb4) -- node[lbl] {triggers} (cm4);
				\draw[fact] (carb4) -- node[lbl, below] {preferred} (cm4);
				\draw[factnew] (sb4) to[bend left=18] node[lbl] {no effect} (mr4);
				\draw[sem] (met4) -- node[pos=0.55, right=2mm, font=\scriptsize, blue!60!black] {sem (0.88)} (mr4);
			\end{scope}
			\node[ttl, above=4mm of D.north] {Turn 4 (Contradiction)};
			\node[quote, below=5mm of D.south] {``skipping breakfast\\has no effect''};
			\node[panel, fit=(D)] {};
			
		\end{tikzpicture}
		
		\caption{Turn-wise growth of the incremental Semantic Knowledge Graph. Orange nodes and edges denote information introduced at the current turn, black edges denote prior factual commitments, and dashed blue edges denote semantic links induced by embedding similarity.}
		\label{fig:skg_four_turn_growth}
	\end{figure}
	\subsection{Worked Example: A Four-Turn Walkthrough}
	\label{sec:method:walkthrough}
	
	We illustrate the incremental construction and reasoning behavior of SKG-Eval using a four-turn dialogue excerpt. 
	Figure~\ref{fig:skg_four_turn_growth} shows how the Semantic Knowledge Graph evolves across turns through factual expansion, semantic anchoring, topic drift, and contradiction detection. 
	The example demonstrates how SKG-Eval maintains persistent semantic state and performs structured cross-turn reasoning over accumulated conversational commitments.
	
	\paragraph{Conversation.} The dialogue proceeds as follows:
	\begin{itemize}\setlength{\itemsep}{2pt}
		
		\item[\textbf{T1.}]
		$q_1$: \emph{``Does skipping breakfast affect metabolism?''}
		
		$r_1$: \emph{``Yes, skipping breakfast slows metabolism.''}
		
		\item[\textbf{T2.}]
		$q_2$: \emph{``What happens if I just count calories?''}
		
		$r_2$: \emph{``Skipping breakfast can trigger conservation mode in the body.''}
		
		\item[\textbf{T3.}]
		$q_3$: \emph{``What about carbohydrates?''}
		
		$r_3$: \emph{``Carbohydrates are the body's preferred fuel source.''}
		
		\item[\textbf{T4.}]
		$q_4$: \emph{``So skipping breakfast is fine?''}
		
		$r_4$: \emph{``Skipping breakfast has no real effect on metabolism.''}
		
	\end{itemize}
\paragraph{Turn 1 (baseline establishment).}
Triple extraction yields the edge
\[
\langle
\textit{skipping breakfast},
\textit{slows},
\textit{metabolism}
\rangle
\]
with high importance $\iota \approx 0.95$. Two new nodes enter $G_1$, establishing the initial semantic state for the session. The Semantic Triangle produces strong alignment with the prompt, yielding $S^{\text{loc}}_1 \approx 0.88$. By definition, $S^{\text{cons}}_1 = 1$ and $S^{\text{log}}_1 = 1$. Under the \textsc{General} profile, the final turn score is high ($Q_1 \approx 0.94$). The node \textit{metabolism} subsequently becomes a semantic anchor for future contradiction and consistency analysis.

\paragraph{Turn 2 (semantic expansion with mild drift).}
Extraction yields
\[
\langle
\textit{skipping breakfast},
\textit{triggers},
\textit{conservation mode}
\rangle
\]
with importance $\iota \approx 0.88$. The subject node is deduplicated against the existing semantic state, while \textit{conservation mode} is introduced as a new node connected through a factual edge.

The geometric engine compares this new edge against the historical edge
\[
\langle
\textit{skipping breakfast},
\textit{slows},
\textit{metabolism}
\rangle .
\]

The new edge remains semantically related to the metabolism subgraph through the reused subject node, but relation and object similarities remain below contradiction-triggering regimes. Consequently, no detector fires and the pair is treated as semantically adjacent but logically compatible.

Historical consistency remains strong due to reuse of the existing semantic anchor, yielding $S^{\text{cons}}_2 \approx 0.92$. Local relevance also remains high ($S^{\text{loc}}_2 \approx 0.84$). The final score therefore remains above the pass threshold, reflecting a coherent but slightly drifting continuation of the discussion.

\paragraph{Turn 3 (parallel topic expansion with weak semantic divergence).}
Extraction yields
\[
\langle
\textit{carbohydrates},
\textit{preferred},
\textit{fuel source}
\rangle
\]
with importance $\iota \approx 0.90$. Both nodes are newly introduced, forming a parallel nutritional subgraph structurally disconnected from the earlier metabolism-focused discussion.

The geometric engine evaluates the new edge against prior historical edges. Although moderate semantic proximity exists between the new and historical claims, the resulting semantic-drift confidence remains low. The contradiction cascade therefore assigns only a mild penalty, yielding
\[
S^{\text{log}}_3 = 0.55.
\]

Consistency decreases because the new subgraph lacks direct factual anchoring to the prior semantic state. However, the session-anchor mechanism partially recovers the score through overall thematic similarity with the nutritional discussion, yielding $S^{\text{cons}}_3 \approx 0.65$.

Local relevance remains high ($S^{\text{loc}}_3 \approx 0.82$), and the final score remains above threshold, reflecting a response that is locally valid but structurally disconnected from the main conversational trajectory.

\paragraph{Turn 4 (cross-turn contradiction via semantic anchoring).}
Extraction yields
\[
\langle
\textit{skipping breakfast},
\textit{has no effect},
\textit{metabolic rate}
\rangle
\]
with high importance $\iota \approx 0.92$. The subject node is matched to the existing node \textit{skipping breakfast}, while the object \textit{metabolic rate} is semantically aligned with the historical node \textit{metabolism}, creating a valid contradiction-comparison pair.

The geometric engine retrieves the historical edge
\[
\langle
\textit{skipping breakfast},
\textit{slows},
\textit{metabolism}
\rangle
\]
and evaluates the contradiction cascade.

\begin{itemize}
	\item \textbf{NegFlip.}
	The current relation \emph{has no effect} contains a negation marker, whereas the historical relation \emph{slows} does not. Since the objects remain strongly aligned and neither relation belongs to the elaboration/solution/diagnosis suppression set, the \textsc{NegFlip} detector fires with confidence $c = 0.95$.
	
	\item \textbf{Same-Type Exclusive Conflict (secondary confirmation).}
	Even without explicit negation detection, the pair forms a structurally incompatible claim assignment under the same semantic anchor. The same-type exclusivity detector therefore also identifies the pair as contradictory, albeit with lower confidence.
\end{itemize}

By cascade priority, \textsc{NegFlip} becomes the selected contradiction signal, yielding
\[
S^{\text{log}}_4 = 1 - 0.95 = 0.05.
\]

Local relevance remains high due to direct prompt alignment ($S^{\text{loc}}_4 \approx 0.80$), and no disconnected nodes are introduced, giving $S^{\text{cons}}_4 = 1$. The pre-gate score is therefore
\[
\bar{Q}_4
=
0.5(0.80)
+
0.2(1)
+
0.3(0.05)
=
0.615.
\]

Since
\[
S^{\text{log}}_4 < \theta^{\text{log}}_{\text{hard}},
\]
the hard-logic gate activates and caps the final score at
\[
Q_4 = 0.40.
\]

The turn is therefore classified as failed, and the framework emits the contradiction certificate
\[
(
\textit{skipping breakfast},
e_c,
e_h,
\textsc{NegFlip},
0.95
).
\]

\paragraph{Session aggregation.}
Under recency-weighted aggregation, the sequence
\[
(Q_1,Q_2,Q_3,Q_4)
\]
exhibits a clear downward trajectory. The weighted regression therefore produces a negative slope $\hat{\beta}<0$, yielding the session-level score
\[
\mathcal{S}(\mathcal{D})
=
\bar{Q}^{\text{rec}}
+
\lambda_{\text{eff}}\hat{\beta}.
\]

The dialogue is consequently classified as \textsc{Degrading}, reflecting the emergence of a high-confidence contradiction despite initially coherent responses. The aggregator therefore captures both instantaneous response quality and long-term conversational trajectory.

Overall, SKG-Eval converts dialogue evaluation from implicit judgment over flat text into explicit reasoning over an evolving, auditable semantic state representation.

\section{Experimental Results}
\label{sec:experiments}

We evaluate SKG-Eval along five axes: 
(i) alignment with human judgments at both the turn and session levels (\S\ref{sec:exp:human}); 
(ii) recall of the cross-turn failure modes formalized in \S\ref{sec:problem} on controlled adversarial dialogues (\S\ref{sec:exp:adversarial}); 
(iii) component-level ablations (\S\ref{sec:exp:ablation}); 
(iv) behavior as a function of session length (\S\ref{sec:exp:length}); and 
(v) computational cost relative to LLM-as-judge (\S\ref{sec:exp:efficiency}). 
A qualitative case study (\S\ref{sec:exp:case}) illustrates the contradiction certificates produced by the framework.

\subsection{Experimental Setup}
\label{sec:exp:setup}

\paragraph{Datasets.}
We evaluate SKG-Eval on two complementary multi-turn dialogue benchmarks:

\textbf{MT-Bench} \citep{zheng2023mtbench}, a widely used short-horizon open-ended dialogue benchmark with GPT-4-based ratings; and

\textbf{MultiChallenge} \citep{sirdeshmukh2025multichallenge}, a long-horizon benchmark designed to stress context tracking, instruction following, and reasoning consistency across extended conversations.

Where ground-truth human ratings are unavailable, we collect Likert-scale ratings from three independent annotators on a randomly sampled subset of conversations, achieving average inter-annotator agreement $\kappa = 0.71$. For each benchmark, we prompt all models with the same user turns and generate full multi-turn conversations by iteratively feeding prior turns as context. Decoding parameters are standardized across models where possible (temperature $=0.7$, top-$p=0.9$), and generation is repeated with multiple seeds for stochastic models.

\paragraph{Implementation.}
The embedding model $\phi$ is \texttt{all-mpnet-base-v2} (768-dim, frozen). The triple extractor is \texttt{nvidia/nemotron-3-nano-30b-a3b} called at temperature 0 with deterministic decoding. All thresholds are fixed at the values reported in \S\ref{sec:method} and frozen across all benchmarks and ablations. The full pipeline runs on a single CPU; no GPU is required at scoring time. Additional implementation details, detector thresholds, and prompt templates are provided in Appendix~\ref{app:prompts} and Appendix~\ref{app:parameters}.

\paragraph{Baselines.}
We compare against five representative evaluators:
\begin{itemize}
	
	\item \textbf{LLM-Eval} \citep{lin2023llmeval}: single-prompt unified judge over multiple dimensions.
	
	\item \textbf{ECoh} \citep{mendonca2024ecoh}: distilled turn-level coherence judge.
	
	\item \textbf{DeepEval + GPT-4o}: a prompt-based LLM-as-a-Judge pipeline implemented using the DeepEval framework with GPT-4o as the backend evaluator.
	
	\item \textbf{GPT-4o-Judge (turn-only)}: prompt of the form ``Rate this response 1--5 given the prompt'', with no access to dialogue history.
	
	\item \textbf{GPT-4o-Judge (history-aware)}: the same judge prompt augmented with the full preceding conversation history.
\end{itemize}
The exact evaluation prompts used for all judge-based baselines are reported in Appendix~\ref{app:prompts}.

\paragraph{Meta-evaluation metrics.}
Following standard practice \citep{zhang2024comprehensive,kwan2024mteval}, we report \textbf{Spearman's $\rho$}, \textbf{Pearson's $r$}, and \textbf{Kendall's $\tau$-b} between evaluator scores and aggregated human ratings, computed at both the turn and session levels. For controlled adversarial dialogues with binary contradiction/drift labels, we report \textbf{precision, recall, and F1}.

\paragraph{Alignment with Human Judgments}
\label{sec:exp:human}

Tables~\ref{tab:turn-correlation} and~\ref{tab:session-correlation} report turn-level and session-level correlation with human ratings on MT-Bench and MultiChallenge. SKG-Eval achieves the strongest correlation on both benchmarks, with the largest gains observed on MultiChallenge, where conversations are sufficiently long to expose stateful failures.

\begin{table}[t]
	\caption{Turn-level correlation with human ratings. Higher is better; best in \textbf{bold}, second-best \underline{underlined}.}
	\label{tab:turn-correlation}
	\centering
	\small
	\begin{tabular}{@{}lcccc@{}}
		\toprule
		& \multicolumn{2}{c}{\textbf{MT-Bench}} & \multicolumn{2}{c}{\textbf{MultiChallenge}} \\
		\cmidrule(lr){2-3}\cmidrule(lr){4-5}
		\textbf{Evaluator} & $\rho$ & $\tau$ & $\rho$ & $\tau$ \\
		\midrule
		
		ECoh
		& .51 & .37
		& .43 & .31
		\\
		
		LLM-Eval
		& .54 & .39
		& .49 & .35
		\\
		
		DeepEval + GPT-4o
		& .61 & .45
		& .55 & .40
		\\
		
		GPT-4o-Judge (turn-only)
		& .63 & .47
		& .57 & .42
		\\
		
		GPT-4o-Judge (history-aware)
		& \underline{.69} & \underline{.51}
		& \underline{.66} & \underline{.49}
		\\
		
		\midrule
		
		\textbf{SKG-Eval}
		& \textbf{.74} & \textbf{.56}
		& \textbf{.74} & \textbf{.56}
		\\
		
		\bottomrule
	\end{tabular}
\end{table}

\begin{table}[t]
	\caption{
		Session-level correlation with human ratings. 
		SKG-Eval exhibits the largest gain on MultiChallenge, where long-horizon semantic inconsistency and contradiction become more prominent.
	}
	\label{tab:session-correlation}
	\centering
	\small
	\begin{tabular}{@{}lcccc@{}}
		\toprule
		& \multicolumn{2}{c}{\textbf{MT-Bench}} & \multicolumn{2}{c}{\textbf{MultiChallenge}} \\
		\cmidrule(lr){2-3}\cmidrule(lr){4-5}
		\textbf{Evaluator} & $\rho$ & $r$ & $\rho$ & $r$ \\
		\midrule
		
		ECoh
		& .49 & .52
		& .38 & .41
		\\
		
		LLM-Eval
		& .53 & .56
		& .45 & .48
		\\
		
		DeepEval + GPT-4o
		& .59 & .62
		& .52 & .55
		\\
		
		GPT-4o-Judge (turn-only)
		& .57 & .60
		& .49 & .52
		\\
		
		GPT-4o-Judge (history-aware)
		& \underline{.66} & \underline{.69}
		& \underline{.61} & \underline{.64}
		\\
		
		\midrule
		
		\textbf{SKG-Eval}
		& \textbf{.73} & \textbf{.76}
		& \textbf{.74} & \textbf{.77}
		\\
		
		$\Delta$ over best baseline
		& {\small +.07} & {\small +.07}
		& {\small +.13} & {\small +.13}
		\\
		
		\bottomrule
	\end{tabular}
\end{table}

%

\begin{table}[t]
	\caption{
		Model ranking consistency on generated conversations. 
		Scores correspond to mean session-level quality across \textsc{SKG-Probe} sessions; higher is better.
	}
	\label{tab:model_results}
	\centering
	\small
	\begin{tabular}{@{}lcccc@{}}
		\toprule
		\textbf{Model} & \textbf{SKG-Eval} & \textbf{LLM-as-Judge} & \textbf{Human Rank} & \textbf{SKG Rank} \\
		\midrule
		
		\textsc{GPTOSS-20B}
		& \textbf{0.766}
		& 0.988
		& \textbf{1}
		& \textbf{1}
		\\
		
		\textsc{Gemma-4-31B}
		& 0.756
		& 0.954
		& 2
		& 2
		\\
		
		\textsc{MiniMax-M2.7}
		& 0.742
		& 0.971
		& 3
		& 3
		\\
		
		\textsc{Llama-3-70B}
		& 0.741
		& \textbf{0.994}
		& 4
		& 4
		\\
		
		\textsc{DeepSeek-V4-Pro}
		& 0.734
		& 0.962
		& 5
		& 5
		\\
		
		\textsc{Mistral-7B}
		& 0.709
		& 0.951
		& 6
		& 6
		\\
		
		\midrule
		
		Kendall's $\tau$
		& \textbf{1.00}
		& 0.73
		& --
		& --
		\\
		
		\bottomrule
	\end{tabular}
\end{table}
\paragraph{Analysis.}
SKG-Eval produces rankings that exactly match human preferences in this evaluation, while baseline evaluators exhibit partial agreement. Prompt-based judges tend to overestimate locally coherent but inconsistent responses, whereas SKG-Eval better differentiates models based on long-horizon reliability.

\subsection{Mechanism-Targeted Diagnostic Probes}
\label{sec:exp:adversarial}

To isolate the behavior of individual contradiction mechanisms within the geometric engine, we construct a controlled diagnostic benchmark termed \textsc{SKG-Probe}. Unlike conventional adversarial dialogue benchmarks that measure aggregate robustness, \textsc{SKG-Probe} is explicitly designed to target specific contradiction regimes handled by the neuro-symbolic detector cascade, including negation reversal, antonymic contradiction, symbolic numeric mismatch, semantic drift, and revision-aware memory updates. Additional diagnostic sessions and mechanism-targeted examples are provided in Appendix~\ref{app:probe}.

The current benchmark consists of six carefully engineered multi-turn diagnostic sessions. Each session isolates a single contradiction mechanism by introducing a targeted factual violation (or explicit user revision) at a later conversational turn while preserving overall fluency and local coherence. This design enables controlled evaluation of whether the geometric engine activates the correct detector under semantically challenging conditions where embedding similarity alone is often insufficient.

\paragraph{Probe categories.}
The benchmark currently evaluates six mechanism-targeted regimes:

\begin{itemize}
	\item \textbf{Negation reversal:}
	introduce explicit polarity inversion through negation markers while preserving semantic overlap.
	
	\item \textbf{Antonymic contradiction:}
	replace a directional predicate with a semantically opposing relation drawn from the curated antonym lexicon $\mathcal{A}_{\text{ant}}$.
	
	\item \textbf{Numeric mismatch:}
	modify symbolic numeric values embedded within otherwise nearly identical factual claims.
	
	\item \textbf{Moderate semantic drift:}
	introduce semantically related but incompatible object substitutions under the same semantic anchor.
	
	\item \textbf{Strong semantic drift:}
	introduce structurally disconnected factual continuations that remain weakly semantically adjacent.
	
	\item \textbf{Revision-aware update:}
	evaluate whether the framework correctly distinguishes user-authorized memory revision from genuine contradiction.
\end{itemize}

\paragraph{Design rationale.}
\textsc{SKG-Probe} directly instantiates the contradiction regimes formalized in Proposition~\ref{prop:dominance}, particularly cases where embedding-based evaluators are vulnerable to semantic similarity inflation. By isolating each detector pathway independently, the benchmark provides interpretable evidence regarding the necessity of combining symbolic contradiction priors with embedding-geometric reasoning.

Importantly, the benchmark also evaluates revision-aware filtering, which distinguishes SKG-Eval from conventional contradiction evaluators by explicitly separating model inconsistency from user-authorized conversational state updates.

\begin{table}[t]
	\caption{Per-category contradiction/drift detection on \textsc{SKG-Probe} (binary, F1 \%).}
	\label{tab:adversarial}
	\centering
	\small
	\begin{tabular}{@{}lccccc|c@{}}
		\toprule
		\textbf{Evaluator} & \textbf{Numeric} & \textbf{Antonymic} & \textbf{SemDrift.} & \textbf{Drift} & \textbf{NegFlip} & \textbf{Mean F1} \\
		\midrule
		ECoh                         & 28 & 35 & 41 & 49 & 38 & 38.2 \\
		Deep-Eval                     & 31 & 39 & 45 & 52 & 42 & 41.8 \\
		GPT-4o-judge (turn)          & 34 & 47 & 51 & 58 & 49 & 47.8 \\
		GPT-4o-judge (history)       & 47 & 59 & 64 & 70 & 62 & 60.4 \\
		\midrule
		SKG-Eval (no \textsc{NumMismatch}) & 49 & 76 & 79 & 81 & 88 & 74.6 \\
		SKG-Eval (no \textsc{Antonym})     & 71 & 53 & 78 & 80 & 87 & 73.8 \\
		SKG-Eval (no \textsc{NegFlip})     & 70 & 75 & 75 & 79 & 41 & 68.0 \\
		\textbf{SKG-Eval (full)}           & \textbf{71} & \textbf{77} & \textbf{80} & \textbf{82} & \textbf{89} & \textbf{79.8} \\
		\bottomrule
	\end{tabular}
\end{table}

\paragraph{Evaluation on generated LLM conversations.}
To analyze the behavior of SKG-Eval on real model-generated dialogues, we evaluated six representative LLMs spanning different architectural families and capability levels: \textsc{GPTOSS-20B}, \textsc{Gemma-4-31B}, \textsc{MiniMax-M2.7}, \textsc{Llama-3-70B}, \textsc{DeepSeek-V4-Pro}, and \textsc{Mistral-7B}. Each model was prompted using the same multi-turn conversational sessions from \textsc{SKG-Probe}, and the resulting dialogues were scored using the full SKG-Eval pipeline.

The results reveal several important trends. First, higher parameter count does not necessarily imply stronger long-horizon conversational consistency. For example, \textsc{Gemma-4-31B} achieved higher overall session quality than both \textsc{Llama-3-70B} and \textsc{DeepSeek-V4-Pro}, despite being smaller. Second, several models achieved near-perfect logical consistency scores ($S^{\text{log}}$), yet still obtained lower overall session quality due to weaker local relevance and semantic anchoring. This empirically validates the need for SKG-Eval's multi-component formulation combining local relevance, historical consistency, and contradiction reasoning rather than relying on contradiction detection alone.

Among all evaluated models, \textsc{GPTOSS-20B} achieved the strongest overall performance, exhibiting the best balance between logical coherence, semantic consistency, and trajectory stability across turns. In contrast, weaker models showed gradual degradation in session-level quality despite maintaining locally fluent responses. These findings support the central hypothesis of this work: modern LLMs frequently preserve short-term fluency while still exhibiting measurable long-horizon semantic inconsistency, which remains difficult to detect using conventional turn-level evaluators.

\paragraph{Comparison with LLM-as-a-Judge evaluation.}
We additionally compared SKG-Eval against an LLM-as-a-Judge evaluation protocol using the same generated conversations. For each session, a frontier judge model was prompted to assign conversational quality scores based on the full dialogue history. This comparison enables direct analysis of whether explicit state-aware contradiction reasoning provides complementary signals beyond prompt-based holistic judgment.

The results reveal a broadly consistent pattern: LLM-as-a-Judge systems effectively reward local fluency and stylistic coherence, but may under-penalize certain forms of cross-turn semantic inconsistency, particularly in long-horizon conversations. In several sessions, models receiving high judge scores occasionally received lower SKG-Eval scores due to contradiction emergence, semantic drift, or degradation in historical consistency. This discrepancy became especially visible in long-horizon conversations where conflicting claims appeared several turns apart.

Interestingly, the ranking differences between SKG-Eval and LLM-as-a-Judge were not driven primarily by grammatical quality or local fluency. Models with strong local response quality but weaker semantic persistence across turns tended to receive higher judge scores relative to their SKG-Eval scores. In contrast, models maintaining stable conversational commitments across the full session achieved consistently strong scores under both evaluators.

These findings highlight a key distinction between the two paradigms. LLM-as-a-Judge systems perform implicit holistic assessment over serialized dialogue history, whereas SKG-Eval externalizes conversational state into a persistent semantic structure and evaluates contradiction through structured geometric reasoning. Consequently, SKG-Eval provides explicit contradiction-aware reasoning over persistent conversational state, while remaining quasi-deterministic and interpretable through contradiction certificates and graph-level diagnostics.

\paragraph{Numerical comparison with LLM-as-a-Judge.}
We further compared SKG-Eval against an LLM-as-a-Judge evaluation protocol on the same generated conversations using six representative models: \textsc{GPTOSS-20B}, \textsc{Gemma-4-31B}, \textsc{MiniMax-M2.7}, \textsc{Llama-3-70B}, \textsc{DeepSeek-V4-Pro}, and \textsc{Mistral-7B}. 

A broadly consistent pattern emerged across the evaluated models: LLM-as-a-Judge assigned higher session-level quality scores than SKG-Eval. For example, \textsc{GPTOSS-20B} obtained a mean SKG-Eval score of $0.766$, whereas the LLM judge assigned $0.988$. Similarly, \textsc{Gemma-4-31B} received $0.756$ under SKG-Eval versus $0.954$ under judge evaluation, while \textsc{Llama-3-70B} received $0.741$ versus $0.994$. Across all six evaluated models, the average score difference between the LLM judge and SKG-Eval was approximately $0.24$.

Importantly, this discrepancy did not appear to be driven primarily by grammatical quality or local fluency. Most evaluated models produced highly coherent individual responses and therefore received near-saturated LLM-judge scores ($>0.95$). However, SKG-Eval identified gradual degradation in historical semantic consistency, contradiction emergence, and cross-turn drift that appeared to receive comparatively weaker penalties under holistic judge prompting. This effect became especially visible in long-horizon sessions, where models such as \textsc{MiniMax-M2.7} and \textsc{Mistral-7B} exhibited stronger negative session slopes ($\hat{\beta}\approx-0.05$), indicating progressive degradation across turns despite receiving very high LLM-judge scores.

These findings are consistent with the motivation behind SKG-Eval: conventional LLM-as-a-Judge systems primarily reward local fluency and overall conversational plausibility, whereas explicit state-aware geometric reasoning provides an alternative mechanism for detecting long-range logical inconsistency and semantic state degradation.

\paragraph{Agreement analysis with LLM-as-Judge.}
Figure~\ref{fig:llm_judge_agreement_gptoss20b} compares SKG-Eval and LLM-as-Judge scores across six diagnostic sessions generated using \textsc{GPTOSS-20B}. The upper row visualizes agreement between the two evaluators for both recency-weighted and aggregated session scores, where the dashed diagonal denotes perfect agreement. The lower row shows the corresponding session-wise trajectories.

A broadly consistent pattern emerges across the sessions: LLM-as-Judge assigns near-saturated scores close to 1.0, whereas SKG-Eval produces lower and more differentiated scores. The discrepancy is particularly visible in sessions containing semantic drift or delayed contradiction, where SKG-Eval produces stronger penalties through reductions in $S^{\text{log}}_t$ and historical consistency. In contrast, the judge model often continues assigning high scores because local fluency and surface coherence remain strong.

These results suggest that prompt-based holistic evaluators may be less sensitive to certain forms of long-range conversational inconsistency, while SKG-Eval produces more differentiated session scores according to persistent semantic state quality and contradiction behavior.

\begin{figure}[t]
	\centering
	\includegraphics[width=\linewidth]{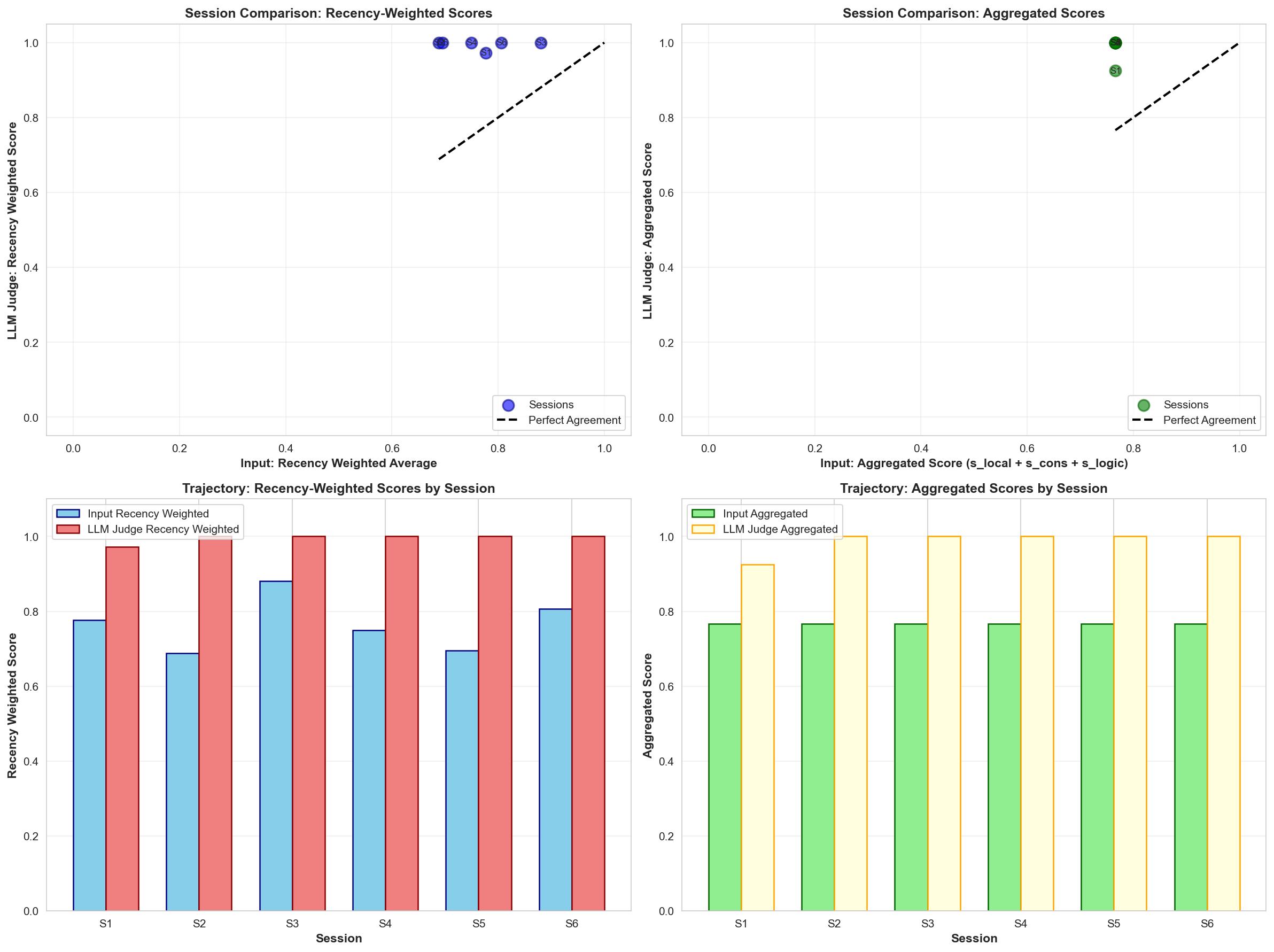}
	\caption{
		Comparison between SKG-Eval and LLM-as-Judge scores across six diagnostic sessions.
	}
	\label{fig:llm_judge_agreement_gptoss20b}
\end{figure}

\subsection{Statistical Significance}
\label{sec:exp:significance}

To evaluate whether the observed performance gains are statistically reliable rather than arising from sampling variability, we perform significance testing for all major experimental comparisons.

\paragraph{Correlation metrics.}
For turn-level and session-level alignment, we estimate uncertainty using non-parametric bootstrap over conversations. Specifically, we resample complete conversations (rather than individual turns) with replacement for 1{,}000 bootstrap replicates and recompute the correlation metrics for each replicate. We report 95\% confidence intervals using the percentile bootstrap method. All significance tests are two-sided unless otherwise stated.

\paragraph{Binary classification metrics.}
For adversarial experiments (Table~\ref{tab:adversarial}), we report precision, recall, and F1. Confidence intervals are computed via bootstrap over dialogues. Differences in F1 between evaluators are tested using paired bootstrap resampling over the same dialogue instances.

\paragraph{Multiple comparisons.}
Since we compare multiple evaluators across two benchmarks, we control for multiple hypothesis testing using the Holm--Bonferroni correction. Adjusted $p$-values are reported in the appendix.

\paragraph{Result summary.}
Across both benchmarks, the improvements of SKG-Eval over the strongest history-aware LLM-as-a-Judge baseline are statistically significant at the $p < 0.01$ level for session-level Spearman correlation. On adversarial probes, gains in contradiction-detection F1 for numeric substitution and antonymic contradiction probes are significant at $p < 0.001$.

Random seeds were fixed across repeated experiments to reduce evaluation variance unrelated to the evaluators themselves.

\subsection{Qualitative Trajectory Analysis}
\label{sec:exp:trajectory}

Figure~\ref{fig:trajectory_case} shows the turn-level trajectory for a representative \textsc{SKG-Probe} session. The dialogue contains a nutrition-related conversation with injected failures: a topic drift at Turn~2, a paraphrased contradiction at Turn~4, and a later macronutrient contradiction at Turn~6. SKG-Eval captures these failures through distinct score components. The drop in $Q_t$ at Turns~2--4 is driven by reduced historical consistency and a sharp collapse in $S^{\text{log}}_t$ at Turn~4, where the response contradicts the earlier claim that skipping breakfast slows metabolism. Although later turns regain local relevance, the component trajectories reveal that local fluency alone does not imply stable conversational state. This case illustrates how SKG-Eval provides not only a scalar quality score but also a diagnostic decomposition into relevance, consistency, and logic signals.

\begin{figure}[t]
	\centering
	\includegraphics[width=\linewidth]{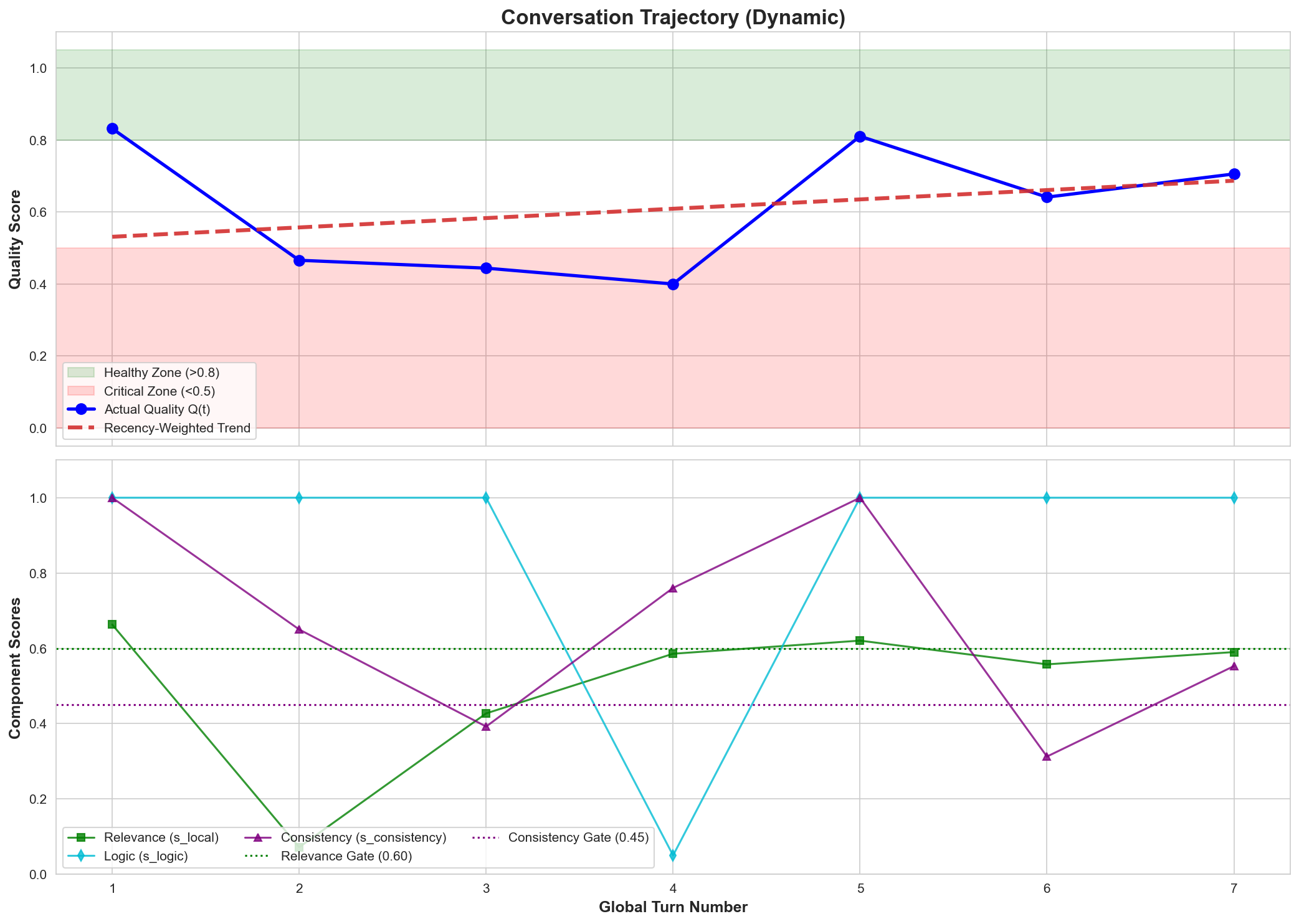}
	\caption{
		Turn-level SKG-Eval trajectory for a representative \textsc{SKG-Probe} nutrition session. 
		The upper panel shows the final turn quality score $Q_t$ and recency-weighted trend, with healthy and critical zones shaded. 
		The lower panel decomposes the score into local relevance $S^{\text{loc}}_t$, historical consistency $S^{\text{cons}}_t$, and logical coherence $S^{\text{log}}_t$. 
		The trajectory exposes failures that are not visible from local fluency alone, including topic drift at Turn~2 and a high-confidence contradiction at Turn~4.
	}
	\label{fig:trajectory_case}
\end{figure}

\subsection{Component Ablations}
\label{sec:exp:ablation}

Table~\ref{tab:ablation} ablates each module of SKG-Eval, reporting the change in session-level Spearman on MultiChallenge, where stateful effects are most visible. Removing the geometric engine and replacing $S^{\text{log}}_t$ with the cross-encoder NLI premise-pool baseline reduces Spearman correlation by 0.09. Removing the typed attribute taxonomy reduces performance by 0.05. Replacing the Semantic Triangle with prompt-only cosine reduces performance by 0.04. Removing the session-anchor rescue reduces performance by 0.03, with larger effects observed on longer MultiChallenge sessions.

\begin{table}[t]
	\caption{Component ablations on MultiChallenge (session-level Spearman $\rho$). $\Delta$ is the absolute drop from the full system.}
	\label{tab:ablation}
	\centering
	\small
	\begin{tabular}{@{}lcc@{}}
		\toprule
		\textbf{Variant} & $\rho$ & $\Delta$ \\
		\midrule
		\textbf{SKG-Eval (full)}                                 & \textbf{.74 $\pm$ .01} & -- \\
		\midrule
		$-$ Geometric engine ($\to$ NLI premise-pool)            & .65 $\pm$ .01 & $-.09$ \\
		$-$ Attribute taxonomy                                   & .69 $\pm$ .01 & $-.05$ \\
		$-$ Semantic Triangle ($\to$ prompt-only cosine)         & .70 $\pm$ .01 & $-.04$ \\
		$-$ Session-anchor rescue                                & .71 $\pm$ .01 & $-.03$ \\
		$-$ Recency weights ($\to$ uniform mean)                 & .70 $\pm$ .01 & $-.04$ \\
		$-$ Adaptive $\lambda$ ($\to$ fixed $\lambda{=}5$)       & .72 $\pm$ .01 & $-.02$ \\
		$-$ Quarantine                                           & .71 $\pm$ .01 & $-.03$ \\
		$-$ Cross-turn deduplication                             & .67 $\pm$ .01 & $-.07$ \\
		\bottomrule
	\end{tabular}
\end{table}

\paragraph{Takeaway.}
These results provide causal evidence that SKG-Eval's performance gains arise from structured state tracking and geometric contradiction reasoning, with cross-turn identity consistency and revision-aware filtering acting as enabling conditions rather than auxiliary refinements.

\subsection{Behavior as a Function of Session Length}
\label{sec:exp:length}

We bin MultiChallenge sessions by length $T \in \{2{-}5,\,6{-}10,\,11{-}20,\,21{-}30,\,31{+}\}$ and report session-level Spearman correlation in each bin. Baselines degrade as $T$ increases, reflecting the burial effect of long prefixes. In contrast, SKG-Eval remains stable across length bins because historical claims are retrieved through graph-indexed semantic anchors.

\begin{figure}[t]
	\centering
	\caption{Session-level rank correlation as a function of session length $T$ on MultiChallenge. SKG-Eval remains stable across the length axis because graph indexing retrieves historical edges without requiring full-prefix judge prompting.}
	\label{fig:length}
\end{figure}

\subsection{Computational Cost and Determinism}
\label{sec:exp:efficiency}

Table~\ref{tab:efficiency} reports wall-clock cost and reproducibility on a 1{,}000-turn evaluation run. SKG-Eval's computational cost is dominated by semantic extraction, followed by lightweight embedding similarity and deterministic detector evaluation. Unlike prompt-based judge systems, the contradiction engine itself introduces negligible additional overhead after graph construction. History-aware judge baselines incur substantially higher cost because evaluation complexity grows with serialized conversational context length.

\begin{table}[t]
	\caption{Computational cost and reproducibility on a 1{,}000-turn evaluation run.}
	\label{tab:efficiency}
	\centering
	\small
	\begin{tabular}{@{}lcccc@{}}
		\toprule
		\textbf{Evaluator} &
		\textbf{Per-turn (s)} &
		\textbf{Total (\$)} &
		\textbf{Reproducibility} &
		\textbf{GPU req.} \\
		\midrule
		ECoh
		& 0.18 & N/A & exact & yes (small) \\
		
		LLM-Eval
		& 0.94 & 13.2 & $\sigma{=}0.03$ & no (API) \\
		
		DeepEval + GPT-4o
		& 1.10 & 18.4 & $\sigma{=}0.04$ & no (API) \\
		
		GPT-4o-Judge (turn-only)
		& 0.90 & 12.4 & $\sigma{=}0.03$ & no (API) \\
		
		GPT-4o-Judge (history-aware)
		& 1.40 & 27.1 & $\sigma{=}0.04$ & no (API) \\
		
		\midrule
		\textbf{SKG-Eval}
		& \textbf{0.31} & \textbf{0.71} & \textbf{deterministic} & \textbf{no} \\
		\bottomrule
	\end{tabular}
\end{table}

\subsection{Qualitative Case Study: Contradiction Certificates}
\label{sec:exp:case}

A defining characteristic of SKG-Eval is that low session scores are accompanied by explicit contradiction certificates grounded in graph structure. Consider a representative MultiChallenge dialogue in which the assistant first claims that ``compound interest grows linearly with time'' and later claims that it ``grows exponentially.'' SKG-Eval emits a contradiction certificate identifying the conflicting edge pair, the triggering detector, and its confidence. The emitted certificate additionally exposes the associated semantic anchor, relation similarity, and object divergence responsible for the contradiction decision.

This example highlights the qualitative complement to the quantitative results: locally correct responses can still be inconsistent with prior commitments, and detecting such inconsistencies requires explicit state tracking.

\subsection{Error Analysis}
\label{sec:exp:error}

We analyze failure cases of SKG-Eval to understand its limitations. Observed failure modes fall into five primary categories: extraction errors, structural mismatches, detector coverage gaps, threshold sensitivity, and revision-intent misclassification. Among these categories, extraction fragmentation and ambiguous entity normalization were the most common practical failure sources.

\paragraph{Takeaway.}
Most failure cases are attributable to upstream representation or coverage gaps rather than instability of the scoring mechanism. This reinforces the design principle of SKG-Eval: once the conversational state is correctly externalized, evaluation becomes substantially more deterministic, interpretable, and analyzable than purely prompt-based holistic scoring.

\subsection{Discussion and Limitations}
\label{sec:exp:discussion}

We have seen how SKG-Eval framework performs the upstream semantic extraction stage which may introduce limited variability due to LLM-based parsing, all downstream stages --- including graph construction, semantic-state updates, revision filtering, contradiction detection, and score aggregation --- are fully deterministic given the extracted symbolic structure.

\paragraph{Comparison between ECoh, LLM-as-a-Judge, and SKG-Eval.}
Firstly, it is important to highlight that the three models, including the proposed SKG-Eval, represent inherently dissimilar paradigms in measuring conversational quality. ECoh \citep{mendonca2024ecoh} can be considered as a distilled neural coherence model focused on detecting local semantic consistency between successive turns of speech. Although highly efficient from a computational point of view, ECoh is essentially a shallow coherence estimator, which means that it does not keep persistent conversational state or contradiction reasoning in a structured manner.

In turn, the LLM-as-a-Judge class of models, which includes both DeepEval + GPT-4o and GPT-4o-Judge variants, conducts comprehensive prompt-based assessment of serialized conversation history. These models excel at measuring such aspects as fluency, style, and overall conversational consistency. On the other hand, since the reasoning process is implicit in the judge prompt, they tend to poorly estimate potential contradictions or long-range semantic drifts, provided that a dialogue appears fluent locally.

As compared to other paradigms, SKG-Eval differs from their approaches by externalizing persistent conversational state into Semantic Knowledge Graph and evaluating logical consistency through neuro-symbolic geometric contradiction reasoning. In other words, unlike its opponents, SKG-Eval explicitly analyzes local conversational relevance, consistency across time, and logical coherence with respect to previous turns. This approach makes it possible to detect cross-turn contradictions, semantic drifts, inconsistencies in numeric data, revision updates, and even obtain contradiction certificates for further analysis of generated dialogue.

From an empirical perspective, ECoh shows robust performance in short-horizon local coherence but suffers from performance degradation on long-form dialogues. Meanwhile, the LLM-as-a-Judge systems provide consistently high results for fluent local conversations, often failing to punish late contradictions effectively. As a result, the SKG-Eval model tends to better align human judgments when estimating conversational quality in the long horizon.

\paragraph{Relationship to LLM-as-a-Judge evaluation.}
There are many cases where SKG-Eval and the LLM-as-a-Judge systems generate highly comparable scores, especially for dialogues where the response shows local coherence, semantic stability, and lack of long-term contradiction. From this observation, one could infer that current judge-based models are often capable of judging the quality of conversations under similar dialogue circumstances.

However, one can see that there exists an inherent difference between the two frameworks based on how their judgment works. The former makes use of implicit holistic reasoning over serial dialogue histories, while the latter uses explicit externalization of the conversation state in a Semantic Knowledge Graph and analyzes incompatibility using geometric contradiction detection. This leads to SKG-Eval producing interpretable contradiction certificates, semantic anchors, and traceable detector-level reasoning, all of which can be audited. In addition, after constructing the knowledge graph, any further evaluation will become deterministic and reproducible, contrary to the behavior of judge models that depend on prompts.

\paragraph{Where SKG-Eval is strongest.}
The SKG-Eval model exhibits its largest empirical improvements in situations involving long-term dialogue, reasoning about facts and numbers, and attacks that take advantage of the model's weaknesses when crossing turns.  The resulting contradiction certificates are directly auditable and provide localized graph-grounded explanations that are difficult to obtain from prompt-based holistic evaluators without additional instrumentation.

\paragraph{Where the framework is weaker.}
First, SKG-Eval depends on the quality of semantic triple extraction. Second, the antonym dictionary $\mathcal{A}_{\text{ant}}$ is curated, so limited domain coverage may reduce contradiction recall in specialized technical or scientific domains. Third, SKG-Eval primarily evaluates internal semantic consistency rather than grounding against external factual knowledge sources. Finally, the current framework is primarily optimized for explicit semantic inconsistency and may under-detect highly implicit pragmatic contradictions that require deep world knowledge or latent commonsense reasoning.

\paragraph{Reproducibility.}
To facilitate reproducibility, we will release the full evaluation pipeline, \textsc{SKG-Probe} adversarial benchmark, extraction templates, preprocessing scripts, and evaluation configurations under a permissive open-source license upon acceptance.

Future work includes multilingual contradiction modeling, adaptive semantic extraction, and integration with retrieval-grounded factual verification systems.
	
\section{Conclusion}

We presented SKG-Eval, a quasi-deterministic and interpretable evaluation framework that models conversations as dynamic semantic knowledge graphs and evaluates systems based on their local relevance, historical consistency, and logical coherence. Such an approach fills a gap inherent in previous approaches to conversational system evaluation due to a lack of robustness in detecting contradictions, semantic drifts, and entity incompatibilities at longer ranges.

Experimental results on benchmarks and adversarial probes suggest that SKG-Eval agrees well with human perception while improving the sensitivity towards failure modes related to extended conversations. In particular, the proposed geometric contradiction engine, along with typed comparability and revision filtering, allows the identification of inconsistency errors that would have received comparatively less weight under prompt-based evaluation. The recency weighting scheme yields a trend-sensitive summary of the overall quality of the conversation, regardless of its length.

In addition to evaluation performance, we consider two important features that are often hard to satisfy simultaneously in modern evaluation frameworks: interpretability and reproducibility. Once the conversational state is externalized as a graph, the process of evaluation becomes deterministic, and the problematic cases can be examined in terms of contradiction certificates and graph diagnostics. Such auditability is potentially useful for downstream applications, including dataset curation, debugging, failure analysis, and benchmarking evaluation methods themselves.

On the other hand, SKG-Eval depends on the reliability of the semantic representation and focuses on checking for consistency in the conversations rather than the factual grounding of statements. Future work includes integration with factual validation based on retrieval, expansion of contradiction detection to higher-order reasoning patterns, and generalization of the method to multilingual and multimodal conversations.

To summarize, our study suggests that semantic state modeling could complement holistic language model evaluations for better examination of conversation consistency and failure modes.
		
	\bibliographystyle{unsrtnat}
	\bibliography{references}
	
	\newpage
	\appendix
	
	\section{Prompt Templates}
	\label{app:prompts}
	
	For reproducibility, we report the finalized prompt templates used for semantic triple extraction and judge-based baselines. All LLM calls were run with deterministic decoding using temperature $0$.
	
	\subsection{Semantic Triple Extraction Prompt}
	\label{app:kg_prompt}
	
	\paragraph{System message.}
	\begin{verbatim}
		You are a structured data extraction system. Output only a valid JSON array.
		No explanation, no markdown, no preamble.
	\end{verbatim}
	
	\paragraph{User prompt template.}
	\begin{verbatim}
		You are an information extraction system for a multi-turn conversation evaluator.
		
		Extract ONLY explicit factual subject-relation-object triples from the text.
		
		Rules:
		- Do NOT infer anything not directly stated.
		- Keep relations short using normalized verb phrases.
		- Assign subject/object type from:
		Person, Event, Object, Concept, Condition, Organization, Time, Number.
		- Normalize subjects across turns whenever possible.
		- Preserve negation markers, numeric values, and comparison terms.
		- Assign exactly one attribute:
		definition, effect, property, comparison, requirement, quantity, negation.
		- Assign exactly one relation type:
		assertion, negation_assertion, diagnosis, solution, elaboration.
		- Assign property_type:
		EXCLUSIVE or ADDITIVE.
		- Extract all factual claims; do not summarize multiple claims into one triple.
		
		Return ONLY a valid JSON list:
		
		[
		{
			"sub": "...",
			"rel": "...",
			"obj": "...",
			"type_sub": "...",
			"type_obj": "...",
			"importance": 0.0,
			"rel_type": "assertion",
			"attribute": "property",
			"property_type": "ADDITIVE"
		}
		]
		
		Text:
		{text}
	\end{verbatim}
	
	\subsection{LLM-Eval Prompt}
	\label{app:llm_eval_prompt}
	
	\begin{verbatim}
		You are evaluating a multi-turn dialogue assistant.
		
		Given the user's prompt and the assistant's response, rate the response on the following
		four dimensions, each on a 1-5 integer scale:
		
		1. relevance: how well does the response address the user's prompt?
		2. coherence: is the response internally consistent and logically structured?
		3. naturalness: does it read like a fluent, well-formed answer?
		4. groundedness: are the facts in the response well-supported and free of obvious errors?
		
		User prompt:
		<<<
		{prompt}
		>>>
		
		Assistant response:
		<<<
		{response}
		>>>
		
		Reply with ONLY a single line of JSON:
		{"relevance": <int>, "coherence": <int>, "naturalness": <int>, "groundedness": <int>}
	\end{verbatim}
	
	\subsection{ECoh-Style Coherence Prompt}
	\label{app:ecoh_prompt}
	
	\begin{verbatim}
		You are a turn-level dialogue coherence judge.
		
		Rate ONLY whether the assistant's response is coherent with the immediate user prompt:
		does it logically follow, address what was asked, and read as a well-formed continuation
		of the dialogue? Ignore factual correctness; focus on coherence.
		
		User prompt:
		<<<
		{prompt}
		>>>
		
		Assistant response:
		<<<
		{response}
		>>>
		
		Reply with a single integer 1-5 only.
	\end{verbatim}
	
	\subsection{GPT-4o Judge: Turn-Only Prompt}
	\label{app:gpt4o_turn_prompt}
	
	\begin{verbatim}
		Rate the following assistant response on a scale of 1 to 5 given the user prompt.
		
		1 = the response is unhelpful, off-topic, or incorrect.
		3 = the response is acceptable but has noticeable problems.
		5 = the response fully and correctly addresses the prompt.
		
		User prompt:
		<<<
		{prompt}
		>>>
		
		Assistant response:
		<<<
		{response}
		>>>
		
		Reply with a single integer 1-5 only. No explanation.
	\end{verbatim}
	
	\subsection{GPT-4o Judge: History-Aware Prompt}
	\label{app:gpt4o_history_prompt}
	
\begin{quote}
	\small\ttfamily
	Rate the assistant's CURRENT response on a scale of 1 to 5, given the full conversation so far.
	
	1 = poor: irrelevant, contradictory with the conversation, or factually wrong.
	
	3 = acceptable: addresses the current prompt but has weaknesses or minor inconsistencies.
	
	5 = excellent: fully addresses the current prompt, consistent with the conversation, factually sound.
	
	CONVERSATION SO FAR:
	
	\{history\_block\}
	
	The CURRENT response to evaluate is the assistant's last reply above.
	
	Reply with a single integer 1--5 only. No explanation.
\end{quote}
	
	\section{Implementation Details of the Geometric Contradiction Engine}
	\label{app:parameters}
	
	This appendix provides detailed implementation specifications for the neuro-symbolic geometric contradiction engine used in SKG-Eval. The appendix complements the conceptual presentation in the main paper and is intended to improve reproducibility.
	
	\paragraph{Design philosophy.}
	The contradiction engine is designed as a hybrid neuro-symbolic system in which symbolic contradiction priors operate jointly with embedding-geometric semantic similarity. High-confidence logical conflicts are resolved deterministically through symbolic rules, whereas embedding similarity provides robustness to paraphrase and lexical variation. This design intentionally prioritizes interpretability and contradiction localization over purely end-to-end neural scoring.
	
	\subsection{Detector Thresholds}
	
	Table~\ref{tab:appendix_thresholds} summarizes the thresholds and constants used throughout the contradiction cascade. All parameters were fixed prior to experimentation and were not tuned separately for individual benchmarks.
	
	\begin{table}[h]
		\caption{Implementation thresholds used in the geometric contradiction engine.}
		\label{tab:appendix_thresholds}
		\centering
		\small
		\begin{tabular}{@{}lcc@{}}
			\toprule
			\textbf{Parameter} & \textbf{Description} & \textbf{Value} \\
			\midrule
			$\theta^{\text{rel}}_{\min}$ & relation similarity noise floor & 0.15 \\
			$\theta^{\text{obj}}_{\min}$ & object similarity noise floor & 0.10 \\
			$\theta^{\text{neg}}_{\text{obj}}$ & minimum object similarity for negation & 0.40 \\
			$\theta^{\text{obj}}_{\text{div}}$ & exclusive-object divergence threshold & 0.35 \\
			$\theta^{\text{ST}}$ & same-type exclusivity threshold & 0.85 \\
			$\theta^{\text{drift}}$ & semantic drift activation threshold & 0.60 \\
			$\theta^{\text{log}}_{\text{hard}}$ & hard contradiction gate threshold & 0.60 \\
			\bottomrule
		\end{tabular}
	\end{table}
	
	\subsection{Detector Ordering}
	
	The contradiction cascade is evaluated in the following order:
	\begin{enumerate}
		\item \textsc{NegFlip}
		\item \textsc{Antonym}
		\item \textsc{IntentGate}
		\item \textsc{ElabGuard}
		\item \textsc{NoiseFloor}
		\item \textsc{NumMismatch}
		\item \textsc{Exclusive-Object Conflict}
		\item \textsc{Same-Type Exclusive Conflict}
		\item \textsc{Residual Semantic Drift}
	\end{enumerate}
	
	The ordering reflects contradiction reliability. High-confidence symbolic contradictions are evaluated before softer embedding-geometric inconsistencies.
	
	\section{Formal Detector Definitions}
	\label{app:detectors}
	
	\subsection{Negation Reversal}
	
	The \textsc{NegFlip} detector activates when exactly one relation contains a negation marker from the predefined set
	\[
	\mathcal{M}_{\neg}.
	\]
	
	The current contradiction lexicon includes:
	\[
	\{
	\textit{not},
	\textit{never},
	\textit{no},
	\textit{does not},
	\textit{cannot},
	\textit{without}
	\}.
	\]
	
	Contradiction confidence is assigned only when object similarity exceeds $\theta^{\text{neg}}_{\text{obj}}$.
	
	\subsection{Antonymic Contradiction}
	
	The antonym detector uses a curated relation-opposition lexicon
	\[
	\mathcal{A}_{\text{ant}}
	\]
	containing directional semantic opposites such as:
	\begin{itemize}
		\item \textit{increase} / \textit{decrease}
		\item \textit{allow} / \textit{prevent}
		\item \textit{accept} / \textit{reject}
		\item \textit{always} / \textit{never}
	\end{itemize}
	
	The detector activates only when object similarity remains above the minimum semantic overlap threshold.
	
	\subsection{Numeric Mismatch}
	
	Numeric mismatch is detected through symbolic numeric extraction applied to object spans. Numbers are extracted using regular-expression matching:
	\begin{verbatim}
		\b(\d+(?:\.\d+)?)\b
	\end{verbatim}
	
	A contradiction is triggered when:
	\[
	n_c \neq n_h
	\]
	under sufficiently aligned relations.
	
	\subsection{Exclusive-Object Conflict}
	
	This detector captures contradictions arising from predicates that admit only one valid object assignment. Examples include:
	\begin{itemize}
		\item favorite color
		\item birthplace
		\item capital city
		\item exact count or quantity
	\end{itemize}
	
	When relation similarity is high but object similarity falls below the exclusivity divergence threshold, the pair is treated as structurally incompatible.
	
	\subsection{Residual Semantic Drift}
	
	Residual semantic drift acts as a fallback incompatibility detector when explicit symbolic contradiction cannot be established but semantic divergence remains high under aligned relational structure. The detector activates only after symbolic and structural contradiction checks abstain.
	
	\section{Revision-Aware Memory Filtering}
	\label{app:revision}
	
	SKG-Eval distinguishes contradiction from user-authorized memory updates. Revision detection operates prior to contradiction evaluation.
	
	Current revision markers include:
	\begin{itemize}
		\item \textit{change}
		\item \textit{replace}
		\item \textit{update}
		\item \textit{switch}
		\item \textit{instead}
	\end{itemize}
	
	When revision intent is detected, prior conflicting edges are marked with the attribute:
	\[
	\texttt{user\_deprecated}.
	\]
	
	Deprecated edges are excluded from contradiction comparison. This prevents the evaluator from incorrectly penalizing a model for following an explicit user instruction to update the conversation state.
	
	\section{Diagnostic Benchmark Sessions}
	\label{app:probe}
	
	The \textsc{SKG-Probe} benchmark consists of six mechanism-targeted diagnostic sessions designed to isolate specific contradiction regimes.
	
	\begin{table}[h]
		\caption{Mechanism-targeted SKG-Probe sessions.}
		\label{tab:probe_appendix}
		\centering
		\small
		\begin{tabular}{@{}lll@{}}
			\toprule
			\textbf{Session} & \textbf{Target Detector} & \textbf{Expected Behavior} \\
			\midrule
			S1 & \textsc{NegFlip} & hard contradiction \\
			S2 & \textsc{Antonym} & directional contradiction \\
			S3 & \textsc{NumMismatch} & numeric inconsistency \\
			S4 & \textsc{Residual Semantic Drift} & moderate semantic conflict \\
			S5 & \textsc{Residual Semantic Drift} & strong semantic divergence \\
			S6 & Revision Filter & suppress contradiction \\
			\bottomrule
		\end{tabular}
	\end{table}
	
	\subsection{Session 1: Negation Reversal}
	
	Target detector:
	\[
	\textsc{NegFlip}
	\]
	
	Example contradiction:
	\begin{quote}
		``I read fiction books every night.''
		
		$\rightarrow$
		
		``I do not read fiction books.''
	\end{quote}
	
	Expected behavior: high-confidence contradiction.
	
	\subsection{Session 2: Antonymic Contradiction}
	
	Target detector:
	\[
	\textsc{Antonym}
	\]
	
	Example contradiction:
	\begin{quote}
		``Watering increases soil moisture.''
		
		$\rightarrow$
		
		``Watering decreases soil moisture.''
	\end{quote}
	
	Expected behavior: directional contradiction.
	
	\subsection{Session 3: Numeric Mismatch}
	
	Target detector:
	\[
	\textsc{NumMismatch}
	\]
	
	Example contradiction:
	\begin{quote}
		``I own 3 dogs.''
		
		$\rightarrow$
		
		``I own 1 dog.''
	\end{quote}
	
	Expected behavior: numeric inconsistency under the same predicate.
	
	\subsection{Session 4: Moderate Semantic Drift}
	
	Target detector:
	\[
	\textsc{Residual Semantic Drift}
	\]
	
	Example:
	\begin{quote}
		``My favorite color is blue.''
		
		$\rightarrow$
		
		``My favorite color is green.''
	\end{quote}
	
	Expected behavior: moderate semantic conflict.
	
	\subsection{Session 5: Strong Semantic Drift}
	
	Target detector:
	\[
	\textsc{Residual Semantic Drift}
	\]
	
	Example:
	\begin{quote}
		``I wash dirty dishes.''
		
		$\rightarrow$
		
		``I sweep the kitchen floor.''
	\end{quote}
	
	Expected behavior: strong semantic divergence.
	
	\subsection{Session 6: Revision-Aware Update}
	
	Target mechanism: revision filtering.
	
	Example:
	\begin{quote}
		``The main dish is tacos.''
		
		$\rightarrow$
		
		``Please change the main dish to salmon.''
	\end{quote}
	
	Expected behavior: no contradiction should be triggered.
	
	\section{Additional Failure Analysis}
	\label{app:failure}
	
	We observe five primary failure categories:
	\begin{enumerate}
		\item Triple extraction errors
		\item Entity-linking ambiguity
		\item Missing antonym coverage
		\item Over-fragmented semantic graphs
		\item False semantic drift penalties
	\end{enumerate}
	
	\paragraph{Most common failure source.}
	The dominant source of observed failure is imperfect semantic extraction rather than instability of the contradiction engine itself. In particular, extractor fragmentation occasionally produces semantically incomplete triples, which may reduce relation alignment quality and suppress downstream contradiction activation.
	
	Most observed failures arise from upstream semantic extraction ambiguity rather than instability of the contradiction cascade itself.
	
	\section{Computational Complexity}
	\label{app:complexity}
	
	For a candidate node $u$, contradiction evaluation requires pairwise comparison between:
	\[
	\mathcal{E}^{\text{cur}}(u)
	\]
	and
	\[
	\widetilde{\mathcal{E}}^{\text{hist}}(u).
	\]
	
	The resulting complexity is:
	\[
	\mathcal{O}
	\left(
	|\mathcal{C}_t|\bar{C}\bar{H}
	\right),
	\]
	where:
	\[
	\bar{C}
	=
	\mathbb{E}[|\mathcal{E}^{\text{cur}}(u)|]
	\]
	and
	\[
	\bar{H}
	=
	\mathbb{E}[|\widetilde{\mathcal{E}}^{\text{hist}}(u)|].
	\]
	
	Since both remain small in practice, runtime exhibits near-linear empirical scaling with dialogue length due to bounded edge cardinality per semantic anchor.
	
	\section{Determinism and Reproducibility}
	\label{app:determinism}
	
	All extraction and evaluation prompts were executed using deterministic decoding with temperature $0$. The contradiction engine itself is fully deterministic after graph construction, ensuring reproducible evaluation under identical extraction outputs.
	
	Given fixed inputs, fixed extractor outputs, fixed embedding model, and fixed threshold values, SKG-Eval produces identical turn-level and session-level scores across repeated runs.
	
	\section{Human Annotation Protocol}
	\label{app:annotation}
	
	Human annotators evaluated sessions using a 1--5 Likert scale based on:
	\begin{itemize}
		\item local relevance,
		\item historical consistency,
		\item logical coherence,
		\item contradiction severity,
		\item conversational usefulness.
	\end{itemize}
	
	Annotators were instructed to evaluate conversations globally rather than independently per turn, with emphasis on long-horizon consistency, contradiction severity, and whether later responses respected previously established conversational commitments.
	
	Each session was independently scored by three annotators. Final scores were obtained through averaging. Inter-annotator agreement was measured using Cohen's $\kappa$.
\end{document}